\documentclass[journal]{IEEEtran}
\usepackage{cite}
\usepackage{amsmath,amssymb,amsfonts}
\usepackage{algorithm,algorithmic}
\usepackage{graphicx}
\usepackage{hyperref}
\hypersetup{hidelinks=true}
\usepackage{textcomp}
\usepackage{xcolor}
\usepackage{wrapfig}
\usepackage{times}
\usepackage{tikz}
\usepackage{mathtools}
\usepackage{url}
\usepackage[normalem]{ulem}
\usepackage[T1]{fontenc}
\usetikzlibrary{automata, positioning, arrows.meta}
\newtheorem{definition}{Definition}
\newtheorem{theorem}{Theorem}
\newtheorem{remark}{Remark}

\newtheorem{example}{Example}
\newtheorem{lemma}{Lemma}

\newenvironment{proof}{\begin{IEEEproof}}{\end{IEEEproof}}

\newcommand{\zt}[1]{{\color{blue} #1}}
\renewcommand{\zt}[1]{#1}

\def\BibTeX{{\rm B\kern-.05em{\sc i\kern-.025em b}\kern-.08em
    T\kern-.1667em\lower.7ex\hbox{E}\kern-.125emX}}
\begin{document}
\title{Optimal Constrained sc\mbox{-}LTL Planning in MDPs via Switching Policies}

\author{Zetong Xuan and Yu Wang, \IEEEmembership{Senior Member, IEEE}
\thanks{This work was supported in part by the Army Research Office (ARO) under Grant W911NF-25-1-0243.}
\thanks{Zetong Xuan and Yu Wang are with the Department of Mechanical and Aerospace Engineering, University of Florida, Gainesville, FL 32611 USA (e-mail: z.xuan@ufl.edu; yuwang1@ufl.edu).}
\thanks{\copyright~2026 IEEE. Personal use of this material is permitted.
Permission from IEEE must be obtained for all other uses, in any current or
future media, including reprinting/republishing this material for advertising
or promotional purposes, creating new collective works, for resale or
redistribution to servers or lists, or reuse of any copyrighted component of
this work in other works. Accepted for publication in the IEEE Transactions on
Automatic Control; DOI: 10.1109/TAC.2026.3709203.}}

\maketitle
\begin{abstract}
We study the synthesis of optimal policies for planning problems on Markov decision processes with both objectives and safety constraints specified in co\mbox{-}safe linear temporal logic (sc\mbox{-}LTL). 
Our problems are inherently non\mbox{-}Markovian due to the complexity of the sc\mbox{-}LTL specification and may require policy randomization to balance the objective and constraint.
We propose a novel approach that reduces the constrained sc\mbox{-}LTL planning problem to a constrained reachability problem on an extended model. 
We then show that a class of switching policies constructed from stationary policies for the individual sc\mbox{-}LTL specifications is sufficient for optimality for the constrained reachability problem. 
Our finding enables a tractable linear program to compute the optimal policy.
A grid world case study demonstrates that our switching policies can achieve the optimal trade\mbox{-}off between the objective and the safety constraint and validates both optimality and tractability.
\end{abstract}


\section{Introduction}
\label{sec:intro}
\IEEEPARstart{M}{odern} autonomous systems often operate under dynamic and time\mbox{-}dependent safety constraints while accomplishing complex logic\mbox{-}based tasks~\cite{baier2008principles,fainekos2005temporal,kressgazit2009temporallogicbased}. 
A fundamental challenge lies in balancing task efficiency with formal safety guarantees, as these objectives may be contradictory~\cite{altman2021constrained}. 
\zt{In many safety\mbox{-}critical applications, such guarantees are naturally posed as probabilistic requirements~\cite{mallett2021progression, aksaray2021probabilistically, jha2018safe, dennisding2014hierarchical}, i.e., enforcing a prescribed bound on risk rather than demanding absolute safety in all executions.}
For instance, drones must periodically return to their base for inspections and battery recharging to ensure operational safety and reliability. 
However, achieving high efficiency in complex tasks like surveillance may justify tolerating a certain level of risk, including the potential loss of drones. 
This motivates the need for planning that maximizes task success subject to formal logic constraints.

We study constrained planning problems in which both the objective and \zt{a constraint specification} are expressed as co\mbox{-}safe linear temporal logic (sc\mbox{-}LTL)~\cite{kupferman2001model}. Here, ``safety'' refers to an operational requirement that is enforced as an inequality constraint with a probabilistic guarantee.
sc\mbox{-}LTL specializes standard linear temporal logic (LTL)~\cite{baier2008principles} to its co\mbox{-}safe fragment, {whose specifications can be witnessed by a finite prefix and are well suited for specifying mission goals and constraint requirements}.
For example, an sc\mbox{-}LTL specification can require sequential reachability, such as ``visit region $A$, then region $B$'' 
or conditional objectives, such as ``if region $C$ is visited, then region $D$ must be visited afterwards''. 
In a variety of planning settings, sc\mbox{-}LTL has been employed as task objectives, including in probabilistic systems such as MDPs and partially observable MDPs~\cite{ulusoy2012incremental,xuan2025control}, robotic planning~\cite{nenchev2016receding}, and hybrid dynamical systems~\cite{bhatia2010motion}.

Solving the constrained sc\mbox{-}LTL planning problem defined on Markov decision processes (MDPs) is challenging because the optimal solution can be a stochastic and non\mbox{-}Markovian policy~\cite{altman2021constrained, ding2014optimal}.
Stochastic policies are often required to balance conflicting objectives (e.g., task success vs.\ safety)~\cite{altman2021constrained}, which contrasts with unconstrained sc\mbox{-}LTL planning where a purely deterministic policy can be optimal.
The non\mbox{-}Markovian nature arises from two sources. 
First, since both the objective and the constraints are specified using sc\mbox{-}LTL specifications that involve complex sequential ordering, the optimal policy must depend on more than just the current state. 
Second, even when both the objective and the constraint reduce to simple reachability tasks, the optimal policy may still be non\mbox{-}Markovian. 
This is because satisfying one sc\mbox{-}LTL specification may depend on a finite prefix that overlaps or conflicts with the prefix required to satisfy the other.

Despite the complexity of the optimal policy, we show that it suffices to consider a simple switching structure.
In particular, an optimal solution for constrained sc\mbox{-}LTL planning can be taken to be a switching policy with three modes:
(1) a mixed policy that balances the objective and safety constraint,
(2) a pure optimistic policy that prioritizes the objective, and 
(3) a pure pessimistic policy that prioritizes safety. 
The mixed policy addresses the inherent trade\mbox{-}off between objective achievement and constraint satisfaction, while the switching mechanism exploits the fact that sc\mbox{-}LTL specifications can be satisfied by a finite prefix of a path. 
Once one specification is satisfied, the policy can switch to focusing exclusively on the other.

Building on this insight into the structure of the optimal policy, we develop a tractable solution pipeline for constrained sc\mbox{-}LTL planning. 
\zt{
After introducing preliminaries in Section~\ref{sec:prelim} and stating the problem formally in Section~\ref{sec:motivation}, we proceed as follows. 
}
First, in Section~\ref{sec:4}, we reduce the problem to a constrained reachability problem on an extended model of the MDP. 
This reduction ensures that both the optimistic and pessimistic policies become stationary deterministic policies in the extended model, making them easy to compute. 
Second, in Section~\ref{sec:switching}, we prove that a switching policy, constructed by switching among the mixed, optimistic, and pessimistic policies (all stationary in the extended model), is sufficient for optimality. 
This result lies outside the scope of existing constrained MDP theory~\cite{altman2021constrained}, whose standard assumptions, such as a discount factor or an absorbing MDP with a single target set, do not hold in our case. 
Finally, in Section~\ref{sec:LP}, we develop an algorithm based on linear programming (LP) to compute the optimal switching policy in the extended model, and in Section~\ref{sec:case} we validate the approach on a grid world case study. 
The policy is then mapped back to the original MDP, yielding the desired non\mbox{-}Markovian optimal policy.

\subsection*{Related work}
Our work addresses constrained sc\mbox{-}LTL planning on finite\mbox{-}state MDPs, without assuming the feasibility of the constraint, any particular dependency between objective and constraint formulation, or structural assumptions on the dynamics. 
In comparison, most existing works on planning with temporal logic constraints either consider temporal logic only as a constraint (not as an objective), assume a fixed dependency between the objective and constraint specifications, or focus on specific dynamical systems~\cite{ulusoy2012incremental, ding2014optimal, smith2011optimal, nenchev2016receding, bhatia2010motion, hammond2021multiagent, bozkurt2021modelfree, jha2018safe, liu2022recurrent}.

Early studies focus on minimizing cost while strictly satisfying a temporal logic specification. 
Representative works include incremental or optimal control for MDPs and multi\mbox{-}agent systems~\cite{ulusoy2012incremental, ding2014optimal, smith2011optimal}, as well as approximate or sampling\mbox{-}based dynamic programming methods for scalability~\cite{li2017samplingbased, li2019approximate, deng2017approximate}. 
These approaches typically assume that the temporal logic specification is feasible and must be satisfied almost surely.

Alternatively, some works consider scenarios where the temporal logic constraints are soft and can be relaxed. 
In this setting, the constraint is either relaxed or replaced with a cost\mbox{-}based surrogate. 
For example, \cite{cai2023optimal} proposes a relaxed product MDP specification that balances cost, satisfaction probability, and violation penalty.
Similarly, several works~\cite{cohen2023temporal, li2019approximate, aksaray2021probabilistically} incorporate cost\mbox{-}based constraints when temporal logic specifications cannot be fully enforced. 
These works do not consider multiple temporal logic specifications at the same time, and thus are less expressive than our formulation.

Other related works consider planning problems where both the objective and the constraint are specified by temporal logic specifications.
However, their satisfaction outcomes are not independent but often semantically coupled, and thus cannot represent cases where the objective and constraint are logically independent. 
\cite{kretinsky2018learningbased} studied model\mbox{-}based reinforcement learning for mean\mbox{-}payoff optimization under $\omega$\mbox{-}regular constraints, which contains LTL specifications. 
\cite{hammond2021multiagent} proposed a model\mbox{-}free reinforcement learning framework that maximizes a weighted sum of satisfaction probabilities over multiple LTL objectives. 
\cite{bozkurt2021modelfree} considered lexicographic preferences among temporal logic specifications, where safety is prioritized over task completion and reward collection.

Finally, there is related research on signal temporal logic planning for continuous\mbox{-}time systems. However, these works typically assume deterministic dynamics and thus cannot be generalized to finite\mbox{-}state MDPs.
Examples include safe control under signal temporal logic constraints~\cite{jha2018safe} and recurrent\mbox{-}set based synthesis~\cite{liu2022recurrent}.

\section{Preliminaries} \label{sec:prelim}

This section introduces preliminaries on labeled Markov decision processes and co\mbox{-}safe linear temporal logic. 

\subsection{Labeled Markov Decision Processes}\label{sec:mdp} 

We use labeled Markov decision processes (LMDPs) to model planning problems where each decision has a potentially probabilistic outcome. 
LMDPs augment standard Markov decision processes \cite{baier2008principles} with state labels, enabling the assignment of properties, such as safety and liveness, to a sequence of states. 

\begin{definition} \label{def:mdp}
A labeled Markov decision process is a tuple $\mathcal{M} = (S, A, T, \beta, \allowbreak \Lambda, L)$ where
\begin{itemize}
    \setlength{\itemsep}{0pt}
    \item $S$ is a finite set of states, 
    \item $A$ is a finite set of actions where $A(s)$ denotes the set of allowed actions in state $s\in S$,
    \item $T: S \times A \times S \to [0,1]$ is the transition probability function such that for all $s\in S$, 
    \[
        \sum_{s'\in S}T(s,a,s') = \begin{cases}
            1, & a \in A(s) \\
            0, & a \notin A(s)
        \end{cases},
    \]
    \item $\beta \in \Delta(S)$ is the initial distribution, here we denote the set of probability distributions on $S$ by $\Delta (S)$,  
    \item $\Lambda$ is a finite set of atomic propositions, 
    \item $L: S \to 2^{\Lambda}$ is a labeling function.
\end{itemize}
\end{definition}

The set of atomic propositions $\Lambda$ and the function $L$ of the model $\mathcal{M}$ enable high\mbox{-}level symbolic descriptions of objectives and constraints that can reflect the context of the planning tasks.
A path of the LMDP $\mathcal{M}$ is an infinite state sequence $\sigma = s_0 s_1 s_2 \cdots$ such that for all $t \ge 0$, there exists $a_t \in A(s_t)$ with $T(s_t,a_t,s_{t+1}) > 0$.
The semantic path corresponding to $\sigma$ is given by $L(\sigma) = L(s_0)L(s_1)\cdots$, derived using the labeling function $L$.
Given a path $\sigma$, the $t$-th state is denoted by $\sigma[t] = s_t$.
We denote the prefix by $\sigma[{:}t] = s_0 s_1\cdots s_t$ and the suffix by $\sigma[t{:}] = s_t s_{t+1}\cdots$.

We introduce several classes of policies to facilitate further discussion.
A policy is a function that maps a finite state\mbox{-}action prefix
$h_t=(s_0,a_0,\ldots,s_t)$ and the current time step $t$ to a distribution over actions in $A(s_t)$.
When the input is restricted to the current state $s_t$ only, we say the policy is \emph{stationary}.
When the input is restricted to the current state $s_t$ and the current time step $t$, we say the policy is \emph{Markovian}.
Otherwise, we say the policy is \emph{non\mbox{-}Markovian}.
Moreover, if the output distribution is always concentrated on a single action, the policy is said to be deterministic. 
Otherwise, we say the policy is stochastic.

Given a policy $\pi$, the interconnection of $\mathcal M$ and $\pi$ induces a closed\mbox{-}loop stochastic process, denoted by $\mathcal M_\pi$, which generates a random infinite state\mbox{-}action sequence $(s_0,a_0,s_1,a_1,\ldots)$.
Specifically, $s_0$ is drawn from the initial distribution $\beta$.
At each time $t\ge 0$, given the finite state\mbox{-}action prefix $h_t=(s_0,a_0,\ldots,s_t)$, an action $a_t$ is drawn from $\pi(h_t,t)$, and the next state $s_{t+1}$ is drawn according to $T(s_t,a_t,\cdot)$.
Note that the induced state sequence is Markovian only when $\pi$ is Markovian.

The above generative rule uniquely determines probabilities of all events about infinite state\mbox{-}action trajectories that can be specified by finite prefixes, and hence also a well\mbox{-}defined probability distribution over infinite state paths by ignoring the actions.
In particular, the probability of any finite state\mbox{-}action prefix
$h_{t+1}=(s_0,a_0,\ldots,s_t,a_t,s_{t+1})$ is defined recursively by
\begin{align}
P_\pi(s_0) &:= \beta(s_0), \notag\\
P_\pi(h_{t+1}) &:= P_\pi(h_t)\,\pi(h_t,t)(a_t)\,T(s_t,a_t,s_{t+1}). \notag
\end{align}
A (sample) path generated by the closed\mbox{-}loop process is denoted by $\sigma \sim \mathcal M_\pi$.

\subsection{Co\mbox{-}safe Linear Temporal Logic and Deterministic Finite Automata} \label{subsec:sc-LTL}
In an LMDP $\mathcal{M}$, whether a given semantic path $L(\sigma)$ satisfies a property such as avoiding unsafe states can be expressed using linear temporal logic.
LTL can specify the change of labels along the path by connecting Boolean variables over the labels with two propositional operators, negation $(\neg)$ and conjunction $(\wedge)$, two temporal operators, next $(\bigcirc)$ and until $(\mathcal{U})$. 

\begin{definition} \label{def:ltl} 
The LTL specification is defined by the syntax
\begin{align}
\varphi ::= {\rm true}\,|\, \alpha \,|\, \varphi_1 \wedge \varphi_2\,|\,\neg \varphi \,|\,\bigcirc \varphi\,|\, \varphi_1\, \mathcal{U} \varphi_2, \quad \alpha \in \Lambda.
\end{align}
Satisfaction of an LTL specification $\varphi$ on a path $\sigma$ of an MDP (denoted by $\sigma \models \varphi$) is defined inductively as:
${\rm true}$ is satisfied on $\sigma$;
$\alpha$ is satisfied on $\sigma$ if $\alpha\in L(\sigma[0])$;
$\varphi_1 \wedge \varphi_2$ is satisfied on $\sigma$ if $\sigma\models \varphi_1$ and $\sigma\models \varphi_2$;
$\neg \varphi$ is satisfied on $\sigma$ if $\sigma\not\models \varphi$;
$\bigcirc \varphi$ is satisfied on $\sigma$ if $\sigma[1{:}] \models \varphi$;
$\varphi_1\, \mathcal{U} \varphi_2$ is satisfied on $\sigma$ if there exists $t$ such that $\sigma[t{:}] \models \varphi_2$ and for all $\tau<t,\,\sigma[\tau{:}]\models \varphi_1$.
\end{definition}

Other propositional and temporal operators can be derived from previous operators, e.g., (or) $\varphi_1 \vee \varphi_2 \coloneqq \neg(\neg \varphi_1 \wedge \neg \varphi_2)$, (eventually) $\lozenge \varphi \coloneqq \mathrm{true}\,\mathcal{U} \varphi$, and (always) $\square \varphi \coloneqq \neg \lozenge \neg \varphi$.

In this work, we focus on the co\mbox{-}safe fragment of the LTL specification, which can be verified via finite prefixes.
We say the LTL specification $\varphi$ is co\mbox{-}safe \cite{lacerda2014optimal} if all the infinite sequences $\sigma \models \varphi$ satisfy:
there exists a truncated finite prefix $\sigma[{:}t]=s_0s_1\cdots s_t$, $t\in \mathbb{N}$ such that $\sigma[{:}t] \circ \sigma' \models \varphi$ for any infinite sequence $\sigma'$, where $\circ$ denotes concatenation.

We use deterministic finite automata (DFAs) as a graph representation of an sc\mbox{-}LTL objective, thus simplifying the complex rule\mbox{-}based objective into reachability. 
Given an sc\mbox{-}LTL objective $\varphi$, we can construct a DFA $\mathcal{A}_\varphi$ (with alphabet $\Sigma=2^{\Lambda}$) such that a path $\sigma \models \varphi$ if and only if $\sigma$ is accepted by the DFA $\mathcal{A}_\varphi$ which has a reachability objective.
\begin{definition}
A deterministic finite automaton is a tuple $\mathcal{A} = (Q, \Sigma,\delta, q_\mathrm{init}, Z)$  where
\begin{itemize}
    \item $Q$ is a finite set of automaton states, 
    \item $\Sigma$ is a finite alphabet, 
    \item $\delta: Q \times  \Sigma \to Q$ is a (partial) transition function, 
    \item $q_\mathrm{init} \in Q$ is the initial automaton state, 
    \item $Z\subseteq Q$ is a set of final automaton states. 
\end{itemize}
\end{definition}
The DFA verifies the path $\sigma$ by its semantic path $L(\sigma)$. 
Each label on the path moves the automaton state through transition $\delta$.  
A path $\sigma$ is accepted by $\mathcal{A}_\varphi$ if and only if there exists $t\in \mathbb{N}$ such that the prefix $\sigma[{:}t]$ moves the initial automaton state $q_\mathrm{init}$ to a final state $q\in Z$. 
Moreover, accepting states are absorbing, consistent with the co\mbox{-}safe semantics that satisfaction is certified by a finite prefix.

We use $ P_{\pi} (\sigma \models {\varphi} )$ to denote the probability {that a path $\sigma$ generated under policy $\pi$ satisfies $\varphi$.}
For a state $s$, we introduce the shorthand $ P_{\pi} (s \models  {\varphi} ) :=  P_{\pi} (\sigma \models {\varphi} \vert\, \sigma [0] = s)$.
\section{Problem Formulation}\label{sec:motivation}

We consider planning problems where both the objective and the safety constraint are specified using co\mbox{-}safe LTL specifications over system paths. Formally, the problem is defined as
\begin{align}
\textbf{Objective:} & \quad \max_{\pi \in \Pi} \; P_{\pi} (\sigma \models \phi )
\label{eq:objective}
\\
\textbf{Constraint:} & \quad P_{\pi} (\sigma \models \psi ) \geq \tau, 
\label{eq:constraint}
\end{align}
where $\Pi$ is the set of all admissible policies.\footnote{This formulation is not equivalent to planning for the single conjunctive specification $\phi\wedge\psi$ in general. Optimizing $P_{\pi}(\sigma\models(\phi \wedge \psi))$ does not necessarily enforce $P_{\pi}(\sigma\models\psi)\ge\tau$.}

This problem is challenging because the optimal policy may be both non\mbox{-}Markovian and stochastic. 
Existing methods, such as dynamic programming and reinforcement learning, which search over the space of Markovian policies, are fundamentally incapable of synthesizing such policies~\cite{altman2021constrained, sutton2018reinforcement}.

We next present a minimal example that demonstrates how, even in simple settings, the optimal policy may require history dependence and randomization, highlighting the need to consider non\mbox{-}Markovian and stochastic policies in the general case.

\begin{example} \label{ex:1}
Consider an MDP with five states and initial distribution $\beta(s_{\mathrm{init},1})=\beta(s_{\mathrm{init},2})=0.5$. 
The transition dynamics are shown in Fig.~\ref{fig:aligned-mdp}.
From $s_{\mathrm{init},1}$, action $ {u}$ leads deterministically to $s_\phi$ (with a self\mbox{-}loop), and action $ {v}$ reaches $s_{\psi,1}$ with probability $0.1$ and returns to $s_{\mathrm{init},1}$ with probability $0.9$. 
From $s_{\psi,1}$, the agent returns deterministically to $s_{\mathrm{init},1}$. 
From $s_{\mathrm{init},2}$, action $ {u}$ leads to $s_\phi$, and action $ {v}$ leads to $s_{\psi,2}$ (with a self\mbox{-}loop). 

We consider a constrained sc\mbox{-}LTL planning problem where both the objective $\phi$ and constraint $\psi$ are reachability properties: $\phi = \lozenge Z_\phi$ and $\psi = \lozenge Z_\psi$ with $Z_\phi = \{s_\phi\}$ and $Z_\psi = \{s_{\psi,1},s_{\psi,2}\}$. 
The constraint threshold is $\tau = 0.9$. 
The resulting problem takes the form:
\begin{align}
\max_{\pi} \; P_{\pi} (\sigma \models \lozenge Z_\phi) \notag \\
\mathrm{s.t.\quad}P_{\pi} (\sigma \models \lozenge Z_\psi ) \geq \tau.
\end{align}

We construct a stochastic, non\mbox{-}Markovian policy $\pi_h$ as follows. 
1) At $s_{\mathrm{init},1}$, the policy selects $ {v}$ until the path visits $s_{\psi,1}$, and chooses $ {u}$ thereafter.
2) At $s_{\mathrm{init},2}$, take action $ {u}$ with probability $0.2$ and $ {v}$ with probability $0.8$. 

Such a policy can satisfy the constraint while achieving the highest possible objective function 
\begin{align}\label{eqn:example}
    P_{\pi_h} (\sigma \models \lozenge Z_\phi) 
    & = 0.5 P_{\pi_h} (\sigma \models \lozenge Z_\phi \vert \sigma[0] = s_{\mathrm{init},1}) \notag \\
    & + 0.5 P_{\pi_h} (\sigma \models \lozenge Z_\phi \vert \sigma[0] = s_{\mathrm{init},2}) \notag \\
    & = 0.5\times 1 + 0.5\times 0.2 = 0.6, \notag \\
    P_{\pi_h} (\sigma \models \lozenge Z_\psi) 
    & = 0.5 P_{\pi_h} (\sigma \models \lozenge Z_\psi \vert \sigma[0] = s_{\mathrm{init},1}) \notag \\
    & + 0.5 P_{\pi_h} (\sigma \models \lozenge Z_\psi \vert \sigma[0] = s_{\mathrm{init},2}) \notag \\
    & = 0.5\times 1 + 0.5\times 0.8 = 0.9. 
\end{align} 

In contrast, starting from $s_{\mathrm{init},1}$, no Markovian policy can guarantee reaching both $Z_\phi$ and $Z_\psi$ with probability one.
This is because a Markovian policy chooses actions at $s_{\mathrm{init},1}$ based only on the current state and the time index, and thus cannot condition on whether $s_{\psi,1}$ has already been visited.
Therefore, whenever $\pi$ assigns positive probability to taking $ {u}$ at $s_{\mathrm{init},1}$ at some time, there remains a nonzero probability that $s_{\psi,1}$ has not been reached yet, and taking $ {u}$ leads to the absorbing state $s_\phi$, after which $Z_\psi$ becomes unreachable.

Similarly, any deterministic policy at $s_{\mathrm{init},2}$ must choose either $ {u}$ or $ {v}$ exclusively. To satisfy the constraint, $ {v}$ must be selected, which contributes nothing to the objective.

\end{example}

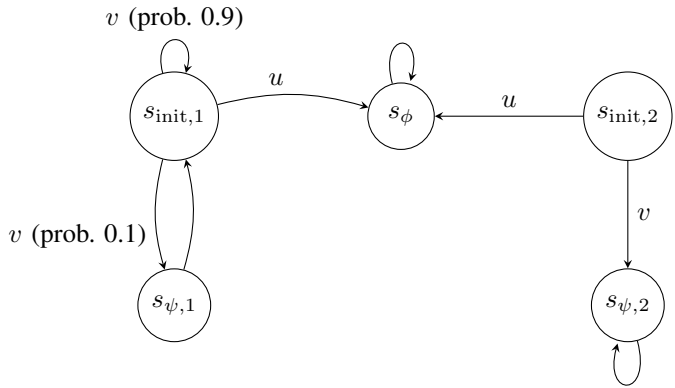
\begin{figure}[h]
    \centering
    \begin{tikzpicture}[>=stealth, node distance=2.5cm and 3cm, on grid, auto]

        \node[state] (S1) {$s_{\mathrm{init},1}$};
        \node[state, right=of S1] (Phi) {$s_\phi$};
        \node[state, below=of S1] (Psi1) {$s_{\psi,1}$};
        \node[state, right=of Phi] (S2) {$s_{\mathrm{init},2}$};
        \node[state, below=of S2] (Psi2) {$s_{\psi,2}$};

        \path[->]
            (S1) edge[bend left=15] node[above left] {$ {u}$} (Phi)
            (S1) edge[bend right=15] node[below left] {$ {v}$ (prob. 0.1)} (Psi1)
            (S1) edge[loop above, looseness=5] node[above] {$ {v}$ (prob. 0.9)} (S1)
            (Phi) edge[loop above] node {} (Phi)
            (Psi1) edge[bend right=15] node[below right] {} (S1);

        \path[->]
            (S2) edge node[above] {$ {u}$} (Phi)
            (S2) edge node[right] {$ {v}$} (Psi2)
            (Psi2) edge[loop below] node {} (Psi2);

    \end{tikzpicture}

    \caption{
    An MDP with five states and two initial states: $s_{\mathrm{init},1}$ and $s_{\mathrm{init},2}$, each selected with probability 0.5.
    From $s_{\mathrm{init},1}$, action $ {u}$ leads deterministically to $s_\phi$ (with a self\mbox{-}loop), while action $ {v}$ reaches $s_{\psi,1}$ with probability 0.1 and loops back to $s_{\mathrm{init},1}$ with probability 0.9.
    From $s_{\psi,1}$, the process returns deterministically to $s_{\mathrm{init},1}$.
    A separate branch begins at $s_{\mathrm{init},2}$, with $ {u}$ leading to $s_\phi$ and $ {v}$ leading to self\mbox{-}looping $s_{\psi,2}$.
    We consider the sc\mbox{-}LTL objective $\max_\pi P_\pi(\sigma \models \lozenge \{s_\phi\})$, subject to the constraint $P_\pi(\sigma \models \lozenge \{ s_{\psi,1}, s_{\psi,2} \}) \ge 0.9$.
    As discussed in Equation~\eqref{eqn:example}, any Markovian or deterministic policy performs worse than the non\mbox{-}Markovian, stochastic policy we construct for this example.}
    \label{fig:aligned-mdp}
\end{figure}
\begin{remark}
Extending our sc\mbox{-}LTL treatment to general LTL is possible but requires replacing the reachability reduction with an $\omega$\mbox{-}regular acceptance condition on the product MDP (e.g., B\"uchi/Rabin/parity), since LTL satisfaction is defined over infinite paths and depends on long-run recurrent behavior.
Accordingly, the analysis must explicitly reason about recurrent classes/accepting end components rather than reachability, and the simple switching-after-satisfaction interpretation used for sc\mbox{-}LTL does not carry over directly~\cite{wagner2025reinforcement}.
\end{remark}
\section{Eliminating Non\mbox{-}Markovianity via Double Product MDP Construction}\label{sec:4}
This section enables the application of solving techniques designed for Markovian policies to the constrained sc\mbox{-}LTL problem by eliminating the planning problem's non\mbox{-}Markovianity. 
Specifically, we show that by transforming the original problem into an equivalent constrained reachability problem on an extended model of the LMDP, a stationary policy becomes sufficient for optimality.
This construction enables us to leverage existing methods, such as linear programming, which are applicable to stationary policies, to solve the transformed problem and subsequently recover 
the optimal policy for the original sc\mbox{-}LTL planning.

The main challenge lies in the fact that, in the original LMDP, even when both the objective and the constraint are expressed as reachability, the optimal policy is not necessarily Markovian, as illustrated earlier in Fig.~\ref{fig:aligned-mdp}.
However, we show that a similar constrained reachability formulation admits a stationary policy to be the optimal policy on a specially constructed product MDP.

\subsection{Double Product MDP Construction}
We adopt the construction of an extended model for the unconstrained LTL planning problem \cite{baier2008principles} to our constrained problem. 
Given the constrained sc\mbox{-}LTL problem with specifications $\phi$ and $\psi$, we construct a double product MDP that encodes the satisfaction status of both specifications. 
This construction transforms both $\phi$ and $\psi$ into two reachability objectives $\lozenge Z_\phi$ and $\lozenge Z_\psi$. 
The double product MDP is defined as follows:
\begin{definition} \label{def:product mdp}
A double product MDP $\mathcal{M}^{\times} = ( S^{\times}, \allowbreak A^{\times} , T^{\times},\beta^{\times}, Z_\phi^\times, Z_\psi^\times)$ of 
an LMDP $\mathcal{M} = (S, A, T, \beta, \allowbreak \Lambda, L)$ and 
a DFA ${\mathcal{A}_\phi}= (Q_\phi,\Sigma_\phi,\delta_\phi,q_{\mathrm{init},\phi},Z_\phi)$ and 
a DFA ${\mathcal{A}_\psi}= (Q_\psi,\Sigma_\psi,\delta_\psi,q_{\mathrm{init},\psi},Z_\psi)$ is defined by 
\begin{itemize}
    \item the set of states $S^\times := S \times Q_\phi \times Q_\psi$, 
    \item the set of actions ${A}^\times := A$ with admissible action set $A^\times(\langle s,q_\phi,q_\psi\rangle)=A(s)$,
    \item one target set $Z_\phi^\times :=\{\langle s,q_\phi, q_\psi \rangle \in S^{\times} |q_\phi\in Z_\phi\}$, 
    \item another target set $Z_\psi^\times :=\{\langle s,q_\phi, q_\psi \rangle \in S^{\times} |q_\psi\in Z_\psi\}$, 
    \item the initial distribution $\beta^\times(\langle s,q_{\mathrm{init},\phi},q_{\mathrm{init},\psi}\rangle):=\beta(s)$ for all $s\in S$, and $\beta^\times(s^\times)=0$ otherwise,
    \item the transition probability function 
\begin{align}
    &T^\times(s^\times ,a, s^{\times\prime}) = \notag \\
    &\begin{cases}
    T(s,a,s') &  \text{if } s^\times = \langle s,q_\phi, q_\psi\rangle , \text{and } \\ 
              &  s^{\times\prime}=\langle s',\delta_\phi(q_\phi,L(s)), \delta_\psi(q_\psi,L(s))\rangle, \\ 
    0 &\text{otherwise}.\notag
\end{cases} 
\end{align} 

\end{itemize}
\end{definition}
This construction captures the transitions of the LMDP, and two parallel DFAs are updated by the labeling on the LMDP. 
As a result, every LMDP path $\sigma$ corresponds uniquely to a double product MDP path $\sigma^\times$. 

\subsection{Sufficiency of Stationary Policy}\label{sec:constrained reachability}
The construction of the double product MDP enables us to formulate the sc\mbox{-}LTL control problem into the following constrained reachability problem,
\begin{align}\label{eq:CRP}
&\max_{\pi^\times} \; P_{\pi^\times}(\sigma^\times \models \lozenge Z_\phi^\times) \notag
\\
\mathrm{s.t.}  \quad &P_{\pi^\times} (\sigma^\times \models  \lozenge Z_\psi^\times) \ge \tau, 
\end{align}

Although reachability in the original MDP may require non\mbox{-}Markovian behavior, we show, perhaps counterintuitively, that a stationary policy is sufficient to achieve optimality in the double product MDP:
\begin{theorem}\label{thm:main}
    A stationary policy is sufficient for optimality in the constrained reachability problem on the double product MDP. (The proof is given in Section~\ref{proof for main}.)
\end{theorem}

This result is enabled by the structure of the double product MDP. 
Each product state $ s^\times = \langle s, q_\phi, q_\psi \rangle $ encodes not only the current LMDP state $ s $, but also the satisfaction status of both sc\mbox{-}LTL specifications via the DFA states $ q_\phi $ and $ q_\psi $. 
The following example shows how the DFA states modify Example~\ref{ex:1}. 

\begin{example}
This example provides a minimal illustration of how extending the original MDP enables a stationary policy to represent the non\mbox{-}Markovian policy used in Example~\ref{ex:1}.
As shown in Fig.~\ref{fig:aligned-mdp}, after visiting $ s_{\psi,1} $, the path returns to the state $ s_{\mathrm{init},1} $ in the original MDP.
Fig.~\ref{fig:product-mdp} illustrates the corresponding transition in the extended model, where after the first visit to $s_{\psi,1}$ the subsequent transition is directed to the distinguishable copy $s_{\mathrm{init},1}'$ rather than $s_{\mathrm{init},1}$.
Here, $s_{\mathrm{init},1}'$ is a distinguishable copy of $s_{\mathrm{init},1}$ that encodes the history information that $s_{\psi,1}$ has already been visited.
Therefore, the non\mbox{-}Markovian policy used in Example~\ref{ex:1} (``at $s_{\mathrm{init},1}$, choose $ {v}$ until the first visit to $s_{\psi,1}$ and choose $ {u}$ thereafter'') can be implemented by a stationary policy on the extended model that chooses $ {v}$ at $s_{\mathrm{init},1}$ and $ {u}$ at $s_{\mathrm{init},1}'$.

For clarity, Fig.~\ref{fig:product-mdp} is only an illustration and does not represent the double product MDP in Definition~\ref{def:product mdp}.
In particular, we omit the explicit automaton states and only show how the augmented state space distinguishes the two cases at $s_{\mathrm{init},1}$ by introducing the distinguishable copy $s_{\mathrm{init},1}'$. 
\end{example}

\begin{figure}[h]
    \centering
    \begin{tikzpicture}[>=stealth, node distance=2.5cm and 3cm, on grid, auto]
        \node[state] (S1) {$s_{\mathrm{init},1}$};
        \node[state, right=of S1] (Phi) {$s_\phi$};
        \node[state, below=of S1] (Psi1) {$s_{\psi,1}$};

        \node[state, below=of Phi, xshift=1.0cm] (S1p) {$s_{\mathrm{init},1}'$};

        \node[state, right=of Phi] (S2) {$s_{\mathrm{init},2}$};
        \node[state, below=of S2] (Psi2) {$s_{\psi,2}$};

        \path[->]
            (S1) edge[bend left=15] node[above left] {$ {u}$} (Phi)
            (S1) edge[bend right=15] node[below left] {$ {v}$ (prob. 0.1)} (Psi1)
            (S1) edge[loop above, looseness=5] node[above] {$ {v}$ (prob. 0.9)} (S1);

        \path[->]
            (Phi) edge[loop above, looseness=5] node {} (Phi);

        \path[->]
            (Psi1) edge[bend left=14] node[pos=0.55, above] {} (S1p)
            (S1p) edge[bend left=10] node[pos=0.55, below] {$ {v}$ (prob. 0.1)} (Psi1);

        \path[->]
            (S1p) edge[bend left=10] node[above right] {$ {u}$} (Phi)
            (S1p) edge[loop below, looseness=5] node[below] {$ {v}$ (prob. 0.9)} (S1p);

        \path[->]
            (S2) edge node[above] {$ {u}$} (Phi)
            (S2) edge node[right] {$ {v}$} (Psi2)
            (Psi2) edge[loop below] node {} (Psi2);

    \end{tikzpicture}
    \caption{
    An extended model of the MDP in Example~\ref{ex:1}.
    Compared with Fig.~\ref{fig:aligned-mdp}, the extended model introduces a new distinguishable state $s_{\mathrm{init},1}'$, which encodes the relevant history information that $s_{\psi,1}$ has been visited.
    As a result, the non\mbox{-}Markovian policy in Example~\ref{ex:1} (choosing $ {v}$ at $s_{\mathrm{init},1}$ until the first visit to $s_{\psi,1}$ and choosing $ {u}$ thereafter) can be implemented by a stationary policy on the extended model that chooses $ {v}$ at $s_{\mathrm{init},1}$ and $ {u}$ at $s_{\mathrm{init},1}'$.
    }
    \label{fig:product-mdp}
\end{figure}
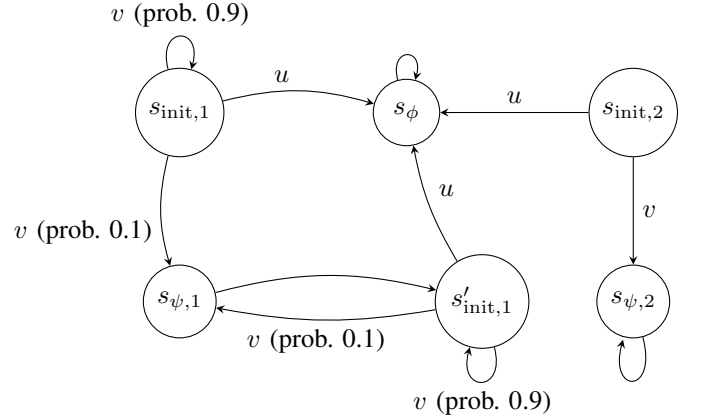

\subsection{Recovering the Optimal Policy on the Original LMDP}
Furthermore, we show that solving the constrained reachability problem on the double product MDP yields a solution to the original constrained sc\mbox{-}LTL planning problem.
\begin{lemma}\label{lem:recover}
    The optimal non\mbox{-}Markovian policy for the constrained sc\mbox{-}LTL problem on the MDP can be recovered from the optimal stationary policy for the constrained reachability problem on the double product MDP.  
\end{lemma}

\begin{proof}
We explicitly construct a policy on the original LMDP from a stationary policy on the double product MDP.
Let $\pi^\times$ be a stationary policy on $\mathcal{M}^\times$.
We define a policy $\pi$ on $\mathcal{M}$ as follows.
Given any finite state\mbox{-}action prefix $h_t=(s_0,a_0,\ldots,s_t)$ on $\mathcal{M}$, we deterministically update the DFA states
$q_{\phi,t}$ and $q_{\psi,t}$ along $h_t$ using the same update convention as in Definition~\ref{def:product mdp},
and set
\begin{align}\label{eq:reconstruct_pi}
\pi(\cdot\mid h_t)
:=
\pi^\times(\cdot \mid \langle s_t,q_{\phi,t},q_{\psi,t}\rangle).
\end{align}

By Definition~\ref{def:product mdp}, each LMDP path $\sigma=(s_0,s_1,\ldots)$ corresponds uniquely to a product path
$\sigma^\times=(\langle s_0,q_{\phi,0},q_{\psi,0}\rangle,\langle s_1,q_{\phi,1},q_{\psi,1}\rangle,\ldots)$.
Moreover, under the reconstruction~\eqref{eq:reconstruct_pi}, $\pi$ and $\pi^\times$ induce the same distribution over corresponding finite prefixes,
since they use the same transition probability function $T$ and the DFA updates are deterministic.
Finally, $\sigma\models \phi$ (resp. $\sigma\models \psi$) holds iff $\sigma^\times$ reaches $Z_\phi^\times$ (resp. $Z_\psi^\times$).
Therefore,
\begin{align}\label{eq:prob_equiv}
P_{\pi}(\sigma \models \phi) &= P_{\pi^\times}(\sigma^\times \models \lozenge Z_\phi^\times) \notag\\
P_{\pi}(\sigma \models \psi) &= P_{\pi^\times}(\sigma^\times \models \lozenge Z_\psi^\times).  
\end{align}

To conclude, let $\pi^\times_{\mathrm{opt}}$ be an optimal stationary solution to the constrained reachability problem on $\mathcal{M}^\times$,
and let $\pi_{\mathrm{opt}}$ be the policy on $\mathcal{M}$ reconstructed from $\pi^\times_{\mathrm{opt}}$ via~\eqref{eq:reconstruct_pi}.
By~\eqref{eq:prob_equiv}, $\pi_{\mathrm{opt}}$ attains the same objective and constraint satisfaction probabilities as $\pi^\times_{\mathrm{opt}}$.

If $\pi_{\mathrm{opt}}$ were not optimal on $\mathcal{M}$, then there would exist a feasible policy $\tilde{\pi}$ on $\mathcal{M}$ with
$P_{\tilde{\pi}}(\sigma\models \psi)\ge \tau$ and $P_{\tilde{\pi}}(\sigma\models \phi) > P_{\pi_{\mathrm{opt}}}(\sigma\models \phi)$.
By~\eqref{eq:prob_equiv}, this would imply the existence of a feasible policy on $\mathcal{M}^\times$
that achieves a strictly larger reachability probability to $Z_\phi^\times$ than $\pi^\times_{\mathrm{opt}}$,
contradicting the optimality of $\pi^\times_{\mathrm{opt}}$.
Therefore, $\pi_{\mathrm{opt}}$ is an optimal policy for the original constrained sc\mbox{-}LTL problem on $\mathcal{M}$.
\end{proof}

With this result, we can apply LP\mbox{-}based techniques or solve the Lagrangian Bellman equation in the double product MDP, and map the resulting stationary policy back to a non\mbox{-}Markovian policy in the original MDP.

For simplicity, we henceforth drop the superscript $\times$ when referring to the double product MDP. 

\section{Optimality of Stationary Policies}\label{sec:switching}

In this section, we establish that a stationary policy suffices to be the solution to the constrained reachability problem on the double product MDP.
At the end of this section, we formally prove Theorem~\ref{thm:main}. 

Although Example~\ref{ex:1} shows that stationary policies are not sufficient to achieve optimality for the constrained reachability problem on a finite\mbox{-}state MDP, we identify a class of non\mbox{-}Markovian policies that can. This class, denoted $\pi_{s \to s}$, switches between different stationary sub\mbox{-}policies depending on which target set (if any) has been reached, thereby resolving the temporal dependency between the two reachability objectives. The switching condition is naturally encoded in the DFA states of the double product MDP, rendering $\pi_{s \to s}$ stationary in that space.

\begin{remark}
Classical constrained MDP theory does not cover our case. 
In the constrained MDP literature (e.g., \cite{altman2021constrained}), similar results typically rely on the fact that the occupancy measure of any arbitrary policy can be matched by that of a stationary policy. 
Since both the objective and constraint are linear in the occupancy measure, this matching implies optimality of stationary policies. 
However, this matching in occupancy measure typically requires either (i) a discount factor or (ii) structural assumptions on the MDP, such as being an absorbing MDP and having a single target set.
Neither assumption holds in our case: we consider undiscounted reachability with two interacting target sets, where behavior after reaching one set still influences the other.
\end{remark}

\subsection{A Switching Policy $\pi_{s\to s}$ for the LMDP}
As illustrated in Fig.~\ref{fig:aligned-mdp}, the optimal policy to the constrained reachability problem may still be non\mbox{-}Markovian due to inherent temporal dependency between the two reachability objectives. 

We propose a class of policies 
\zt{
that consist of a \emph{pre\mbox{-}switch} policy and two \emph{post\mbox{-}switch} policies, where the switch is triggered by the first entrance into $Z_\phi\cup Z_\psi$, 
}
\begin{itemize}
    \item $\pi_{\mathrm{mix}}$ is a stationary policy applied before any states in $Z_\phi$ or $Z_\psi$ are visited. We also refer to it as the pre\mbox{-}switch policy.
    \item $\pi_{\psi}$ is a (post\mbox{-}switch) stationary policy applied once a state in $Z_\phi$ is visited, it then encourages reaching $Z_\psi$,  
    \item $\pi_{\phi}$ is a (post\mbox{-}switch) stationary policy applied once a state in $Z_\psi$ is visited, it then encourages reaching $Z_\phi$, 
\end{itemize}
On the double product MDP, if the initial state is not in $Z_\phi$ or $Z_\psi$, the path begins under $\pi_{\mathrm{mix}}$ and switches to either $\pi_\phi$ or $\pi_\psi$ depending on which target is visited first, as shown in Fig.~\ref{fig:switch-policy}. 
We denote this switching policy by $\pi_{s \to s}$.
It is not necessary to specify the policy after both target sets have been visited, since at that point both $\phi$ and $\psi$ can be seen as satisfied.

\begin{figure}[h]
    \centering
    \begin{tikzpicture}[
        >=Stealth,
        node distance=3cm and 2cm,
        on grid,
        auto,
        rounded corners=4pt,
        policy/.style={rectangle, draw, rounded corners=3pt,
                       minimum width=20mm, minimum height=8mm, align=center}
    ]
        \node[policy] (Mix) {$\pi_{\mathrm{mix}}$};
        \node[policy, right=of Mix, yshift=12mm] (Psi) {$\pi_{\psi}$};
        \node[policy, right=of Mix, yshift=-12mm] (Phi) {$\pi_{\phi}$};

        \draw[->, thick] (Mix.north) |- node[pos=0.3, above left] {$Z_\phi$} (Psi.west);
        \draw[->, thick] (Mix.south) |- node[pos=0.3, below left] {$Z_\psi$} (Phi.west);
    \end{tikzpicture}
    \caption{Illustration of the switching policy $\pi_{s\to s}$. 
    The pre\mbox{-}switch policy $\pi_{\mathrm{mix}}$ is followed until a target set is reached.
    \zt{
    Entering $Z_\phi$ triggers a switch to the post\mbox{-}switch policy $\pi_\psi$, while entering $Z_\psi$
    triggers a switch to the post\mbox{-}switch policy $\pi_\phi$.
    }}
    \label{fig:switch-policy}
\end{figure}
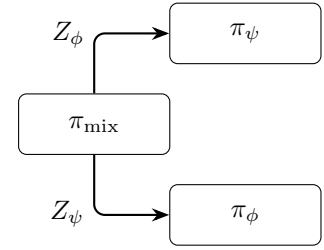

Importantly, this switching policy is realizable as a stationary policy on the double product MDP. 
Because $Z_\phi$ and $Z_\psi$ represent the double product MDP states whose DFA components encode which (if any) target set has been visited. 
Specifically, under an arbitrary policy $\pi_h$, once $Z_\phi$ (resp. $Z_\psi$) is entered, the path will never visit states outside of $Z_\phi$ (resp. $Z_\psi$). 
For all states $\langle s,q_\phi, q_\psi \rangle$, the switching policy $\pi_{s \to s}$ is defined as
\begin{align}
&\pi_{s \to s}(\langle s,q_\phi, q_\psi \rangle) = \notag \\
&
\begin{cases}
\pi_{\mathrm{mix}}(\langle s,q_\phi, q_\psi \rangle), & \text{if } q_\phi \notin Z_\phi \text{ and } q_\psi \notin Z_\psi \\
\pi_\psi(\langle s,q_\phi, q_\psi \rangle),         & \text{if } q_\phi \in Z_\phi \text{ and } q_\psi \notin Z_\psi \\
\pi_\phi(\langle s,q_\phi, q_\psi \rangle),         & \text{if } q_\phi \notin Z_\phi \text{ and } q_\psi \in Z_\psi. \\
\end{cases}
\end{align}
The switch from $\pi_{\mathrm{mix}}$ to either $\pi_\phi$ or $\pi_\psi$ will only happen once. 
And there will be no switch from either $\pi_\phi$ or $\pi_\psi$ back to $\pi_{\mathrm{mix}}$. 
The reason is that once a path has visited $Z_\phi\cup Z_\psi$, it will never enter $S\setminus(Z_\phi\cup Z_\psi)$ again.

To formally show that, without loss of optimality, one may restrict attention to the proposed switching policies for the constrained reachability problem, we introduce the notion of a dominating class of policies.
\begin{definition}
A class of policies $\overline{U}$ is said to be a dominating class of policies for the constrained planning problem for a given initial distribution $\beta$, 
if for any policy $\pi$ there exists a policy $\overline{\pi} \in \overline{U}$ such that
\begin{align}
    &P_{\overline{\pi}}(\sigma \models \lozenge Z_\phi ) \ge P_{\pi} (\sigma \models \lozenge Z_\phi ) \notag \\
    &P_{\overline{\pi}}(\sigma \models \lozenge Z_\psi ) \ge P_{\pi} (\sigma \models \lozenge Z_\psi ).
\end{align}
\end{definition}
        
\begin{theorem}\label{thm: switch 1}
    The class of our proposed switching policy $\pi_{s\to s}$ is a dominating class of policies for the constrained reachability problem on the LMDP. 
\end{theorem}
\begin{proof}
See Section~\ref{sec:thm2}.
\end{proof}

To prove Theorem~\ref{thm: switch 1}, we start from an arbitrary policy $\pi_h$ and
transform it into a switching policy $\pi_{s\to s}$ without decreasing either
the objective or the constraint satisfaction probability.
The transformation consists of two steps.

First, in Section~\ref{sec:step_1}, we modify $\pi_h$ only after the first entrance into $Z_\phi\cup Z_\psi$.
This yields a switching policy (denoted $\pi_{h\to s}$ in Section~\ref{sec:step_1}) whose post\mbox{-}switch component is stationary and
does not decrease either reachability probability.

Second, in Section~\ref{sec:thm2}, we show that the \emph{pre\mbox{-}switch} component of $\pi_{h\to s}$
can be replaced by a stationary policy $\pi_{\mathrm{mix}}$ without decreasing either reachability probability.
Combining the two steps yields a switching policy of the form $\pi_{s\to s}$, establishing the dominance claim.

\subsection{
Stationary Replacement of the Post\mbox{-}Switching Component without Loss of Optimality}\label{sec:step_1}

We first show that, after the first entrance into $Z_\phi\cup Z_\psi$, switching to a stationary policy that maximizes the reachability probability of the remaining target set
only improves, and never worsens, the objective and constraint reachability probabilities.

\zt{Given an arbitrary policy $\pi_h$, we construct a switching policy $\pi_{h\to s}$ as follows:}
\begin{itemize}
    \item Use $\pi_h$ until visiting any state in $Z_\phi$ or $Z_\psi$.
    \item Switch to $\pi_{\psi}\in\arg\max_{\pi} P_{\pi}(\sigma \models \lozenge Z_\psi)$ after entering $Z_\phi$.
    \item Switch to $\pi_{\phi}\in\arg\max_{\pi} P_{\pi}(\sigma \models \lozenge Z_\phi)$ after entering $Z_\psi$.
\end{itemize}

\begin{lemma}\label{lem: switch2}
    The class of our proposed switching policy $\pi_{h\to s}$ is a dominating class of policies for the constrained reachability problem.
\end{lemma}

The proof proceeds by decoupling the interaction between the two reachability specifications.
We show that the objective function can be decomposed into two components:  
(1) the occupancy measure over paths that have not yet reached either $Z_\phi$ or $Z_\psi$, and  
(2) the conditional reachability probability starting from states within $Z_\psi$.
By replacing the second term with the maximum unconstrained reachability probability to $Z_\phi$, we strictly increase (or preserve) the objective function, without affecting the constraint function.
An analogous argument applies when analyzing the constraint function with respect to $Z_\psi$.
The key quantity in this decomposition is the \emph{occupancy measure}, which describes how often each state\mbox{-}action pair is visited before hitting either target set.  

\begin{definition}
Let $\pi$ be a policy and let $Z \subseteq S$ be a target set.
In this paper, we set $Z := Z_\phi \cup Z_\psi$ and define
\begin{align}
S^{0} &\coloneqq \left\{ s \in S \;\middle\vert\; P_{\max}(s \models \lozenge Z)=0 \right\} , \\
S^{?} &\coloneqq S \setminus \left(Z \cup S^{0}\right),
\end{align}
where $P_{\max}(s\models\lozenge Z):=\max_{\pi}P_\pi(s\models\lozenge Z)$ and $P_{\pi}(s \models \lozenge Z):=P_{\pi}(\sigma \models \lozenge Z \mid s_0=s)$. 

(1) \zt{For all $j\in S^?$ and $a \in A(j)$, the probability that the closed\mbox{-}loop process under $\pi$ stays within $S^{?}$ for $t$ steps and visits the state\mbox{-}action pair $(j,a)$ at time $t$, conditioned on $s_0=i$ is given by}
\begin{align}
& \xi_\pi(i,j,a,t,S^{?})
\coloneqq \notag \\
& P_{\pi}\Big(
s_t=j,\ a_t=a,\ s_n\in S^{?}\ \forall n<t
\ \Big|\ s_0=i
\Big). 
\end{align} 

(2) \zt{For all $j\in S^?$ and $a \in A(j)$, the \emph{occupancy measure} of a state\mbox{-}action pair $(j,a)$ on $S^{?}$ is defined as the expected number of times $(j,a)$ is visited:}
\begin{align}
\mu_\pi(j,a)
\coloneqq
\sum_{i\in S^{?}} \sum_{t=0}^\infty \beta(i)\, \xi_\pi(i,j,a,t,S^{?}).
\notag
\end{align}

(3) For all $j\in S^?$, the \emph{state occupancy measure} on $S^{?}$ is then:
\begin{align}
\mu_\pi(j) := \sum_{a \in A(j)} \mu_\pi(j,a).
\end{align}
\end{definition}
\noindent
The restriction to $S^{?}$ is without loss of generality.
Any path that enters $S^{0}$ contributes $0$ to both reachability events, thus it suffices to define the pre\mbox{-}switch occupancy measure on $S^{?}$.

\hspace*{1em}\emph{Proof of Lemma~\ref{lem: switch2}: }
\zt{Let $Z:=Z_\phi\cup Z_\psi$ and recall the sets $S^{0}$ and $S^{?}$ defined above.}
\zt{Since $P_{\max}(s \models \lozenge Z)=0$ for all $s\in S^{0}$, any path that enters $S^{0}$ cannot reach $Z$ under any policy.}
\zt{Therefore, states in $S^{0}$ contribute $0$ to both $P_{\pi}(\sigma\models\lozenge Z_\phi)$ and $P_{\pi}(\sigma\models\lozenge Z_\psi)$ and can be ignored below.}

Given an arbitrary policy $\pi_h$, the reachability probability to the target set $Z_\phi$
can be decomposed as
\begin{align}
& P_{\pi_h}(\sigma \models \lozenge Z_\phi) =
\sum_{i\in Z_\phi}\beta(i)
+\sum_{i\in Z_\psi\setminus Z_\phi}\beta(i)\,P_{\pi_h}(i \models \lozenge Z_\phi) \notag\\
&+\sum_{i\in S^{?}}\beta(i)\,P_{\pi_h}(i \models \lozenge Z_\phi) \notag\\
&=
\sum_{i\in Z_\phi}\beta(i)
+\sum_{i\in Z_\psi\setminus Z_\phi}\beta(i)\,P_{\pi_h}(i \models \lozenge Z_\phi) \notag\\
&+
\sum_{i\in S^{?}}\sum_{a\in A(i)}
\Bigg(
\sum_{j\in Z_\phi}\mu_{\pi_h}(i,a)\,T(i,a,j)  \notag\\
&+\sum_{j\in Z_\psi\setminus Z_\phi}\mu_{\pi_h}(i,a)\,T(i,a,j)\,\eta_{\pi_h}(j)
\Bigg),
\label{eq:decomp_obj}
\end{align}
where for all $j\in Z_\psi\setminus Z_\phi$,
\[
\eta_{\pi_h}(j)
:=
P_{\pi_h}\!\Big(\sigma \models \lozenge Z_\phi \ \Big|\ 
\text{$j$ is the first state visited in $Z$}\Big).
\]

The first equality partitions the event $\sigma\models\lozenge Z_\phi$ according to the initial state.
The second equality further conditions on the first entrance into $Z$: starting from $i\in S^{?}$, the path either enters $Z_\phi$ directly, or enters $Z_\psi\setminus Z_\phi$ first at some $j$ and then reaches $Z_\phi$ with probability $\eta_{\pi_h}(j)$.

For any policy $\pi_h$, there exists a switching policy $\pi_{h \to s}$, such that 
\begin{align}
    P_{\pi_{h \to s}}(\sigma \models \lozenge Z_\phi)  \ge P_{\pi_h}(\sigma \models \lozenge Z_\phi). 
\end{align}
This inequality holds by the construction of $\mu_{\pi_{h \to s}}(s,a)$, 
\begin{align}\label{eqn:ineq}
     &{ P_{\pi_{h \to s}}(\sigma \models \lozenge Z_\phi) } \notag \\
      &= \sum_{i\in Z_\phi} \beta(i) 
       + \sum_{i\in Z_\psi\setminus Z_\phi} \beta(i)\, P_{\max}(i \models \lozenge Z_\phi) \notag \\
      &+ \sum_{i\in S^?} \sum_{a \in A(i)} \sum_{j\in Z_\phi} \mu_{\pi_h}(i, a)\, T(i,a,j) \notag \\
     &+ \sum_{i\in S^?} \sum_{a \in A(i)} \sum_{j\in Z_\psi\setminus Z_\phi} \mu_{\pi_h}(i, a)\, T(i,a,j)\, P_{\max}(j \models \lozenge Z_\phi) \notag \\
     &\ge { P_{\pi_{h}}(\sigma \models \lozenge Z_\phi) }.
\end{align}
Here $P_{\max}(j \models \lozenge Z_\phi)$ denotes the maximal reachability probability from state $j$ to $Z_\phi$.
{
Since $j\in Z_\psi\setminus Z_\phi$ implies that $\psi$ is already satisfied, no constraint remains.
This maximum is then attained by a stationary policy, a classical result for unconstrained reachability problems on finite\mbox{-}state MDPs~\cite{baier2008principles}.
}

Applying the same argument to $P_{\pi_h}(\sigma \models \lozenge Z_\psi)$ yields
\begin{align}
    P_{\pi_{h \to s}}(\sigma \models \lozenge Z_\psi)  \ge P_{\pi_h}(\sigma \models \lozenge Z_\psi).
\end{align}
The two inequalities for the objective and the constraint hold simultaneously for one specific $\pi_{h\to s}$. 
The reason is $\pi_\psi$ does not affect $P_{\pi_{h \to s}}(\sigma \models \lozenge Z_\phi)$ and $\pi_\phi$ does not affect $P_{\pi_{h \to s}}(\sigma \models \lozenge Z_\psi)$. 

By definition, the class of switching policies $\pi_{h\to s}$ is a dominating class of policies for the constrained reachability problem.
\hfill $\blacksquare$

\subsection{
Stationary Replacement of the Pre\mbox{-}Switching Component without Loss of Optimality
}\label{sec:thm2}
In this subsection, we show that the pre\mbox{-}switching (possibly non\mbox{-}Markovian) component in $\pi_{h\to s}$
can be replaced by a stationary policy $\pi_{\mathrm{mix}}$ without loss of optimality.

The argument proceeds in three steps.
\begin{itemize}
    \item \zt{Starting from $\pi_h$, we construct another (possibly non\mbox{-}Markovian) policy $\pi_h'$ such that it does not lose optimality in terms of either reachability probability.}
    \item Using the pre\mbox{-}switch occupancy measure induced by $\pi_h'$ on $S^{?}$, we define a stationary policy $\pi_{\mathrm{mix}}$ on $S^{?}$ via normalization.
    \item We then show that replacing the pre\mbox{-}switch component of $\pi_{h'\to s}$ by $\pi_{\mathrm{mix}}$ does not lose optimality.
\end{itemize}

\begin{lemma}\label{lem:pre_switch_phprime}
Given an arbitrary (possibly non\mbox{-}Markovian) pre\mbox{-}switch policy $\pi_h$ on $S^{?}$, there exists a (possibly non\mbox{-}Markovian) policy $\pi_h'$ such that 
\begin{enumerate}
    \item $\mu_{\pi_h'}(x)<\infty$ for all $x\in S^{?}$, 
    \item the associated switching policies satisfy
    \begin{align}\label{eq:dominate_ph_to_phprime}
    &P_{\pi_{h'\to s}}(\sigma\models\lozenge Z_\phi)\ \ge\ P_{\pi_{h\to s}}(\sigma\models\lozenge Z_\phi)  \\
    &P_{\pi_{h'\to s}}(\sigma\models\lozenge Z_\psi)\ \ge\ P_{\pi_{h\to s}}(\sigma\models\lozenge Z_\psi).
    \end{align}
\end{enumerate}
\end{lemma}

\begin{proof}
Let $Z:=Z_\phi\cup Z_\psi$. Consider any subset $C\subseteq S^{?}$ such that (i) $C\cap Z=\emptyset$ and
(ii) under $\pi_h$, once the path enters $C$ it stays in $C$ forever with probability $1$.
For such a $C$, paths that enter $C$ never reach $Z$ and hence contribute $0$ to both reachability events.

Since $C\subseteq S^{?}=S\setminus(Z\cup S^0)$, we have $P_{\max}(x\models\lozenge Z)>0$ for all $x\in C$.
Define $\pi_h'$ to coincide with $\pi_h$ until the first entrance into $C$. Once the path is in $C$ and $Z$ has not been reached,
$\pi_h'$ will maximize the probability to visit a state in $Z$.
Therefore, under $\pi_h'$ the probability of staying in $C$ forever is $0$.
Applying this modification to every such $C$ eliminates the possibility of staying in $S^{?}$ forever with positive probability,
hence $\mu_{\pi_h'}(x)<\infty$ for all $x\in S^{?}$.

Finally, $\pi_h'$ differs from $\pi_h$ only on paths that would be trapped in some $C$ and thus have zero contribution to both reachability events under $\pi_{h\to s}$.
The modification only creates additional chances to reach $Z$, and the post\mbox{-}switch components are identical by construction,
therefore~\eqref{eq:dominate_ph_to_phprime} holds.
\end{proof}

Now we construct a stationary policy $\pi_\mathrm{mix}$ on $S^{?}$ based on $\pi_h'$.
Since $\pi_\mathrm{mix}$ is stationary, it induces a Markov chain on $S$.
For all $x\in S^{?}$ with $\mu_{\pi_h'}(x)>0$, define
\begin{align}\label{eqn: oc}
\pi_\mathrm{mix}(x,a)=\frac{\mu_{\pi_h'}(x,a)}{\mu_{\pi_h'}(x)} .
\end{align}
If $\mu_{\pi_h'}(x)=0$, then $x$ is never visited before entering $Z$ under $\pi_h'$
(hence $\mu_{\pi_h'}(x,a)=0$ for all $a$), so we may choose $\pi_\mathrm{mix}(x,\cdot)$ arbitrarily.
Moreover, by Lemma~\ref{lem:pre_switch_phprime}, $\mu_{\pi_h'}(x)<\infty$ for all $x\in S^{?}$,
and thus the case $\mu_{\pi_h'}(x)=\infty$
never occurs in~\eqref{eqn: oc} on $S^{?}$.

\begin{lemma}
\label{thm: oc}
Under the construction of $\pi_\mathrm{mix}$ in~\eqref{eqn: oc}, for all $x\in S^{?}$ and $a\in A(x)$, $\mu_{\pi_\mathrm{mix}}(x,a) = \mu_{\pi_h'}(x,a)$. 
(The proof is given in the Appendix.)
\end{lemma}

Therefore, the pre\mbox{-}switch occupancy measure on $S^{?}$ induced by $\pi_h'$ can be replicated by the stationary policy $\pi_\mathrm{mix}$.

\begin{lemma}\label{lem:switch2 to switch1}
Given an arbitrary policy $\pi_h$ and the switching policy $\pi_{h\to s}$ constructed from $\pi_h$ in Section~\ref{sec:step_1},
let $\pi_h'$ be the policy obtained in the first step of this subsection and
let $\pi_{h'\to s}$ be the switching policy constructed from $\pi_h'$ in the same way as $\pi_{h\to s}$.
There exists a switching policy $\pi_{s \to s}$ such that
\begin{align}
&P_{\pi_{s \to s}}(\sigma \models \lozenge Z_\phi)
\ \ge\
P_{\pi_{h' \to s}}(\sigma \models \lozenge Z_\phi)  \notag\\
&P_{\pi_{s \to s}}(\sigma \models \lozenge Z_\psi)
\ \ge\
P_{\pi_{h' \to s}}(\sigma \models \lozenge Z_\psi).
\end{align}
\end{lemma}
\begin{proof}
By Lemma~\ref{lem:pre_switch_phprime}, replacing the pre\mbox{-}switch component of $\pi_{h\to s}$ by $\pi_h'$ cannot decrease either reachability probability.

We now construct $\pi_{s\to s}$ by taking the same post\mbox{-}switch policies as $\pi_{h'\to s}$ and
replacing its pre\mbox{-}switch policy $\pi_h'$ on $S^{?}$ with the stationary policy $\pi_{\mathrm{mix}}$ defined by~\eqref{eqn: oc}.

By Lemma~\ref{thm: oc}, $\mu_{\pi_{\mathrm{mix}}}(i,a)=\mu_{\pi_h'}(i,a)$ for all $i\in S^{?}$ and $a\in A(i)$.
Using the same decomposition argument as in~\eqref{eqn:ineq} (with $\pi_h$ replaced by $\pi_h'$), we obtain
\begin{align} &{ P_{\pi_{s \to s}}(\sigma \models \lozenge Z_\phi) } \notag \\ 
&= \sum_{i\in Z_\phi} \beta(i) + \sum_{i\in Z_\psi\setminus Z_\phi} \beta(i) P_{\max}(i \models \lozenge Z_\phi) \notag \\ 
&+ \sum_{i\in S^{?}} \sum_{a \in A(i)} \sum_{j\in Z_\phi} \mu_{\pi_\mathrm{mix}}(i, a) T(i,a,j) \notag \\ 
&+\sum_{i\in S^{?}} \sum_{a \in A(i)} \sum_{j\in Z_\psi\setminus Z_\phi} \mu_{\pi_\mathrm{mix}}(i, a) T(i,a,j) P_{\max}(j \models \lozenge Z_\phi) \notag \\ 
&\zt{= P_{\pi_{h' \to s}}(\sigma \models \lozenge Z_\phi)}\notag \\ 
&\ge P_{\pi_{h \to s}}(\sigma \models \lozenge Z_\phi). 
\end{align}
The same result holds for the constraint function 
\begin{align} 
P_{\pi_{s \to s}}(\sigma \models \lozenge Z_\psi) \ge P_{\pi_{h \to s}}(\sigma \models \lozenge Z_\psi). 
\end{align} 
\end{proof}

Furthermore, we can show that for any $\pi_h$, there always exists a $\pi_{s\to s}$ that is at least as good as $\pi_h$ in performance. 

\hspace*{1em}\emph{Proof of Theorem~\ref{thm: switch 1}: }
    Given an arbitrary non\mbox{-}Markovian policy $\pi_h$, by Lemma~\ref{lem: switch2}, the switching policy $\pi_{h\to s}$ is always at least as good as the policy $\pi_h$. 
    By Lemma~\ref{lem:switch2 to switch1}, the switching policy $\pi_{s\to s}$ is no worse than the policy $\pi_{h\to s}$. 
    Thus, the switching policy $\pi_{s\to s}$ is a dominating class of policies for the constrained reachability problem. 
\hfill $\blacksquare$

\subsection{Optimality of Stationary Policy}\label{proof for main}
Now we complete the proof for our main result, Theorem~\ref{thm:main}. 
\hspace*{1em}\emph{Proof of Theorem~\ref{thm:main}: }
Theorem~\ref{thm:main} contains two claims on the double product MDP $\mathcal M^\times$:
(i) every proposed switching policy $\pi_{s\to s}$ is a stationary policy on $\mathcal M^\times$,
and (ii) restricting the constrained reachability problem~\eqref{eq:CRP} on $\mathcal M^\times$ to stationary policies is without loss of optimality.

Recall that each product state is $s^\times=\langle s,q_\phi,q_\psi\rangle$.
By the definition of $\pi_{s\to s}$, the action selection depends only on whether
$q_\phi\in Z_\phi$ and/or $q_\psi\in Z_\psi$, i.e., only on the current product state $s^\times$,
and does not depend on the past history.
Therefore, any switching policy of the form $\pi_{s\to s}$ is a stationary policy on $\mathcal M^\times$.

For optimality, by Theorem~\ref{thm: switch 1}, the class of switching policies of the form $\pi_{s\to s}$ is dominating for~\eqref{eq:CRP}:
for any feasible policy, there exists a switching policy of the form $\pi_{s\to s}$ that remains feasible and achieves no smaller objective value.
Hence, an optimal solution of~\eqref{eq:CRP} exists among policies of the form $\pi_{s\to s}$.
Since every such policy is stationary on $\mathcal M^\times$ (shown above), an optimal solution can be chosen stationary, proving the claim.
\hfill $\blacksquare$

\section{Linear Programming Formulation for Constrained sc\mbox{-}LTL Planning}\label{sec:LP}

In this section, we present a linear programming formulation for solving the constrained sc\mbox{-}LTL planning problem. 
As established in Section~\ref{sec:constrained reachability}, the original constrained sc\mbox{-}LTL problem can be transformed into an equivalent constrained reachability problem on the double product MDP, for which an optimal stationary policy exists. 
Our LP formulation computes this optimal stationary policy by combining a pre\mbox{-}switching policy $\pi_{\mathrm{mix}}$ with two stationary sub\mbox{-}policies: an optimistic policy $\pi_{\phi}$ that prioritizes the objective, and a pessimistic policy $\pi_{\psi}$ that prioritizes safety. 
The overall procedure, from problem transformation to LP solution and policy composition, is summarized in Algorithm~\ref{alg:scLTL}.
Here, we omit the details of how to obtain $\pi_{\psi}$ and $\pi_{\phi}$ since it is a well-known result \cite{baier2008principles}. 

\subsection{Checking Feasibility and Removing Self\mbox{-}Loops}

This subsection corresponds to Step~7--14 of Algorithm~\ref{alg:scLTL}.
We first check feasibility of the constraint with threshold $\tau$ using the unconstrained reachability computation in Step~8:
the optimal value of Step~8 equals the maximum achievable probability of reaching $Z_\psi^\times$.
If this value is smaller than $\tau$, then the constrained sc\mbox{-}LTL planning problem is infeasible and Algorithm~\ref{alg:scLTL} terminates and returns \texttt{Infeasible}.

Next, we partition the state space to avoid states with infinite occupancy measures caused by self\mbox{-}loops.
We reuse the state\mbox{-}space partition $(Z,S^{0},S^{?})$ introduced in Section~\ref{sec:switching} (see also Step~11--14 of Algorithm~\ref{alg:scLTL}) on the double product MDP:
\begin{align}
Z &\coloneqq Z_\phi^\times \cup Z_\psi^\times, \notag \\
S^{0} &\coloneqq \left\{ s^\times \in S^\times \;\middle\vert \; P_{\max}(s^\times \models \lozenge Z) = 0 \right\},  \notag \\
S^{?} &\coloneqq S^\times \setminus \left(Z \cup S^{0} \right).  \notag
\end{align}
The set $S^{0}$ can be found via a backward reachability graph search. A path starting from these states can never visit states $Z$ under any policy. 
We exclude $\mu(s,a)$ with $s\in S^{0}\cup Z$ from the decision variable $X$. 

\subsection{Representing the Constrained Reachability Problem in a Linear Form}
This subsection corresponds to Step $15$--$21$ of Algorithm~\ref{alg:scLTL}, where we solve the LP for the pre\mbox{-}switching policy $\pi_{\mathrm{mix}}$. 
Solving for $\pi_{\mathrm{mix}}$ requires introducing a decision variable $X$, from which we define two linear functions $R_\phi(X)$ and $C_\psi(X)$ representing the objective and constraint satisfaction probabilities, respectively. 
This allows us to formulate the constrained reachability problem in standard LP form:
\begin{align}
    &\max_X R_\phi(X) \notag \\
    \mathrm{s.t.}  \quad & C_\psi (X) - \tau \ge 0 \notag \\
    &X \ge 0.
\end{align}

For a stationary policy $\pi$ with $\pi(s)\in \Delta(A(s))$ for all $s\in S$, the objective probability can be written as the following linear function of the occupancy measures:
\begin{align} 
&{ P_{\pi_{s \to s}}(\sigma \models \lozenge Z_\phi) } \notag \\ 
&= \sum_{i\in Z_\phi} \beta(i) + \sum_{i\in Z_\psi\setminus Z_\phi} \beta(i) P_{\max}(i \models \lozenge Z_\phi) \notag \\ 
&+ \sum_{i\in S^{?}} \sum_{a \in A(i)} \sum_{j\in Z_\phi} \mu(i, a) T(i,a,j) \notag \\ 
&+\sum_{i\in S^{?}} \sum_{a \in A(i)} \sum_{j\in Z_\psi\setminus Z_\phi} \mu(i, a) T(i,a,j) P_{\max}(j \models \lozenge Z_\phi).
\end{align}
Meanwhile, the constraint function is expressed as
\begin{align} 
&{ P_{\pi_{s \to s}}(\sigma \models \lozenge Z_\psi) } \notag \\ 
&= \sum_{i\in Z_\psi} \beta(i) + \sum_{i\in Z_\phi\setminus Z_\psi} \beta(i) P_{\max}(i \models \lozenge Z_\psi) \notag \\ 
&+ \sum_{i\in S^{?}} \sum_{a \in A(i)} \sum_{j\in Z_\psi} \mu(i, a) T(i,a,j) \notag \\ 
&+\sum_{i\in S^{?}} \sum_{a \in A(i)} \sum_{j\in Z_\phi\setminus Z_\psi} \mu(i, a) T(i,a,j) P_{\max}(j \models \lozenge Z_\psi) \notag \\
&\ge \tau. 
\end{align}
Furthermore, as in~\cite{derman1970finite}, the occupancy measure must satisfy the following flow constraint 
\begin{align}
\beta(j)
+ 
\underbrace{
\sum_{i \in S^?} 
\sum_{a \in A(i)} 
\mu(i, a)T(i,a,j)  
}_{\text{inflow}}
=
\underbrace{\sum_{a \in A(j)} \mu(j, a)}_{\text{outflow}} .
\end{align}

\subsection{Composing the Switching Policy}
Finally, once $\pi_{\mathrm{mix}}$, $\pi_{\psi}$, and $\pi_{\phi}$ are computed, we compose the switching policy $\pi_{s \to s}$ as in Step $22$--$27$ of Algorithm~\ref{alg:scLTL}.

\begin{algorithm}
\caption{\small Solving Constrained sc\mbox{-}LTL Planning}
\begin{algorithmic}[1]\label{alg:scLTL}
    \small
    \STATE \textbf{Input:} LMDP $\mathcal{M}$, sc\mbox{-}LTL specifications $\phi$, $\psi$, constraint threshold $\tau$
    \STATE \textbf{Output:} Optimal policy $\pi_{s \to s}^\times$ for the constrained sc\mbox{-}LTL planning problem on the double product MDP $\mathcal{M}^\times$

    \vspace{0.5em}
    \STATE \textit{// Construct double product MDP}
    \STATE Construct DFAs $\mathcal{A}_\phi$, $\mathcal{A}_\psi$ corresponding to $\phi$, $\psi$
    \STATE Construct the double product MDP $\mathcal{M}^{\times} = \mathcal{M} \times \mathcal{A}_\phi \times \mathcal{A}_\psi$
    \STATE Define target sets $Z_\phi^\times$, $Z_\psi^\times$ in the double product MDP

    \vspace{0.5em}
    \STATE \textit{// Solve unconstrained reachability subproblems}
    \STATE Compute $\pi_\psi^\times := \arg\max_{\pi^\times} P_{\pi^\times}(\sigma^\times \models \lozenge Z_\psi^\times)$
    \STATE Compute $\pi_\phi^\times := \arg\max_{\pi^\times} P_{\pi^\times}(\sigma^\times \models \lozenge Z_\phi^\times)$
    \STATE \textbf{if} $\max_{\pi^\times}{P_{\pi^\times}(\sigma^\times \models \lozenge Z_\psi^\times)} < \tau$ \textbf{then return} \texttt{Infeasible}

    \vspace{0.5em}
    \STATE \textit{// Partition state space}
    \STATE Compute $Z := Z_\phi^\times \cup Z_\psi^\times$
    \STATE Let $S^0 := \{s^{\times} \in S^\times \mid P_{\max}(s^\times \models \lozenge Z) = 0\}$ 
    \STATE Let $S^? := S^\times \setminus (Z \cup S^0)$

    \vspace{0.5em}
    \STATE \textit{// Solve LP for pre\mbox{-}switching policy}
    \STATE Define decision variable $\mu(i,a)$ for $i \in S^?$
    \STATE Add flow constraints for all $j \in S^?$:
    \STATE \quad $\sum_a \mu(j,a) = \beta(j) + \sum_{i,a} \mu(i,a) T(i,a,j)$
    \STATE Add constraint satisfaction: $C_\psi(\mu) \ge \tau$
    \STATE Maximize objective: $R_\phi(\mu)$
    \STATE Solve LP to obtain $\mu^*$ and corresponding stationary policy $\pi_{\text{mix}}^\times$

    \vspace{0.5em}
    \STATE \textit{// Compose switching policy}
    \STATE Define $\pi_{s \to s}^\times(\langle s, q_\phi, q_\psi\rangle)$ as:
    \STATE \quad \textbf{if} $q_\phi \notin Z_\phi$ and $q_\psi \notin Z_\psi$ \textbf{then} $\pi^\times_{\text{mix}}$
    \STATE \quad \textbf{else if} $q_\phi \in Z_\phi$ and $q_\psi \notin Z_\psi$ \textbf{then} $\pi^\times_\psi$
    \STATE \quad \textbf{else if} $q_\phi \notin Z_\phi$ and $q_\psi \in Z_\psi$ \textbf{then} $\pi^\times_\phi$

    \STATE \textbf{return} $\pi^\times_{s \to s}$
\end{algorithmic}
\end{algorithm}
\section{Case Study}\label{sec:case}

To demonstrate how our method balances objective satisfaction and safety constraint, we construct a constrained planning scenario where satisfying both specifications requires non\mbox{-}trivial trade\mbox{-}off behavior. 
The experiment is conducted on a $4\times 5$ grid world with stochastic dynamics and strategically placed barriers. 

\subsection{Grid Environment and sc\mbox{-}LTL Objective}

\begin{figure}[t]
    \centering
    \includegraphics[width=0.55\linewidth]{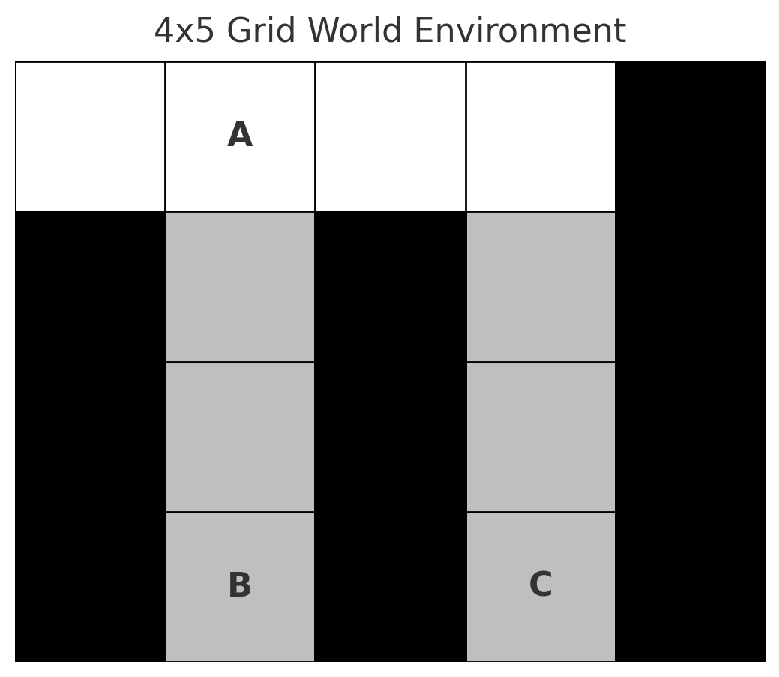}
    \caption{$4\times5$ grid world environment.
    White cells have deterministic transitions. Gray cells have stochastic transitions (vertical moves may slip left/right into adjacent barriers). Black cells are absorbing barriers.
    Letters $A$, $B$, and $C$ indicate labeled target cells.
    The gray corridor is adjacent to barriers, so any policy that attempts to reach $B$ or $C$ incurs a strictly positive probability of absorption. 
    Therefore, in the planning problem where reaching $B$ and reaching $C$ are specified separately as the objective and the constraint, respectively, the inherent trade\mbox{-}off must be explicitly accounted for.
}
    \label{fig:gridworld}
\end{figure}

To evaluate how our method balances objective satisfaction and safety constraints, we design a grid world environment where the two sc\mbox{-}LTL specifications are intentionally in probabilistic conflict, as illustrated in Fig.~\ref{fig:gridworld}.
The key idea is to place absorbing \emph{barrier states} in locations that force any policy pursuing one specification to incur a non\mbox{-}negligible risk of failing the other. 
As a result, no policy can satisfy both specifications with probability $1$. 
Achieving a better success rate for one inevitably reduces the probability of satisfying the other, thereby creating a setting where a trade\mbox{-}off is necessary.
\begin{itemize}
    \item \textbf{Dynamics:} A $4\times5$ grid world.
          Actions $\mathcal A=\{\mathrm N,\mathrm S,\mathrm E,\mathrm W\}$ attempt a unit move to the adjacent cell in the nominated direction; if the adjacent cell lies outside the grid, the agent remains in place.
          On \emph{white} states, all transitions are deterministic (no slip).
          On \emph{gray} states, each action goes to its intended successor with probability $0.8$ and to each orthogonal successor with probability $0.1$.
          The initial position is $(0,0)$.
    
    \item \textbf{Objective:} $\phi = \lozenge (A \land \lozenge B)$,  
          requiring the agent to visit a cell labeled $A$ and subsequently a cell labeled $B$.  

    \item \textbf{Safety constraint:} $P_{\pi}(\sigma\models\psi)\ge\tau$, where $\tau\in[0,1]$ is a user\mbox{-}specified safety level and $\psi=\lozenge C$ requires that the agent visit a cell labeled $C$ at least once.

    \item \textbf{Labels:}  
          $A$ at $(0,1)$, $B$ at $(3,1)$, and $C$ at $(3,3)$.  

    \item \textbf{Barrier states (black cells):}
          $(0,4)$, $(1,0)$, $(1,2)$, $(1,4)$, $(2,0)$, $(2,2)$, $(2,4)$, $(3,0)$, $(3,2)$, $(3,4)$.
          These barriers are absorbing: once entered, the agent cannot leave, and both $\phi$ and $\psi$ are violated.
\end{itemize}

This spatial arrangement forces the trade\mbox{-}off: reaching $C$ reliably requires traversing a risky corridor adjacent to barriers, which can jeopardize the sequence $A \rightarrow B$.  
Conversely, choosing a safer route to $A$ and $B$ bypasses $C$.

\subsection{Experiment}

We evaluate the approach on the conflict\mbox{-}rich $4{\times}5$ grid world introduced above. 
All experiments were implemented in Python on a Windows 11 machine equipped with an Intel i9-14900K processor.
The Mona package \cite{fuggitti-ltlf2dfa} was used to translate the sc\mbox{-}LTL objectives into DFAs. 

The experiment follows the pipeline in Algorithm~\ref{alg:scLTL}. 
Specifically:
\begin{itemize}
    \item $\mathcal{M}$ is the $4\times 5$ grid world with barrier states.
    \item $\phi = \lozenge (A \land \lozenge B)$, DFA size $|\mathcal{A}_\phi|=3$.
    \item $\psi = \lozenge C$, DFA size $|\mathcal{A}_\psi|=2$.
    \item The constraint threshold $\tau \in [0,1]$ is varied to control the desired safety level.
\end{itemize} 
Following Algorithm~\ref{alg:scLTL}, we first compute the unconstrained sub\mbox{-}policies $\pi_\phi$ and $\pi_\psi$ by value iteration, and then, for each threshold $\tau\in[0,1]$, solve the linear program to obtain the mixed (pre\mbox{-}switching) policy $\pi_{\mathrm{mix}}$, which is finally composed with the sub\mbox{-}policies into the switching policy $\pi_{s\to s}$.

A key structural finding is that $\pi_{\mathrm{mix}}$ randomizes at a single state only, namely at state $(0,1)$, while all other visited states remain deterministic.
Moreover, outside the state $(0,1)$, the chosen actions are invariant with respect to $\tau$.
Fig.~\ref{fig:mix_and_subpolicies}(a)--(b) visualize $\pi_{\mathrm{mix}}$ at two representative thresholds ($\tau=0.20$ and $0.60$). 
The green arrows indicate the two outgoing actions being mixed, with line width proportional to the mixing weights.
As $\tau$ increases, the mixing probabilities at $(0,1)$ smoothly shift to satisfy the constraint more, while the red arrows elsewhere (deterministic actions) are unchanged.
Fig.~\ref{fig:mix_and_subpolicies}(c)--(e) show the sub\mbox{-}policies $\pi_\phi$ and $\pi_\psi$ used to construct $\pi_{s\to s}$.

Fig.~\ref{fig:tradeoff-lp} shows the trade\mbox{-}off between the objective function and constraint function obtained by sweeping $\tau$ with step $0.01$.
The instance is feasible up to $\tau=0.64$. 
When $\tau>0.64$ no policy can satisfy the constraint. 
For small thresholds (approximately $\tau\le 0.21$), the resulting policy $\pi_{s\to s}$ coincides with the unconstrained deterministic policy $\pi_\phi$.
As $\tau$ increases toward the feasibility limit, the solution reallocates probability mass to improve $P(\sigma\models\psi)$ at the expense of $P(\sigma\models\phi)$, producing the expected trade\mbox{-}off.

\begin{figure}[t]
    \centering
    \begin{minipage}[t]{0.48\linewidth}
      \centering
      \includegraphics[width=\linewidth]{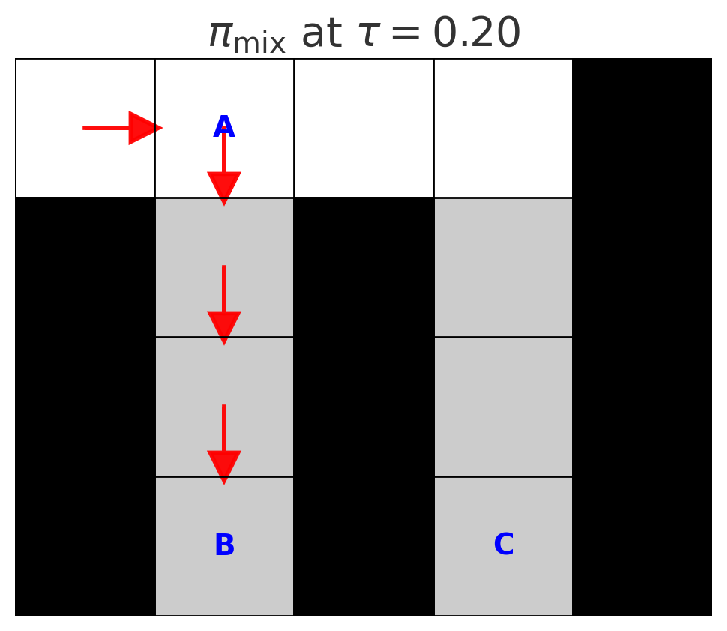}
      \par\vspace{2pt}\footnotesize (a)
    \end{minipage}\hfill
    \begin{minipage}[t]{0.48\linewidth}
      \centering
      \includegraphics[width=\linewidth]{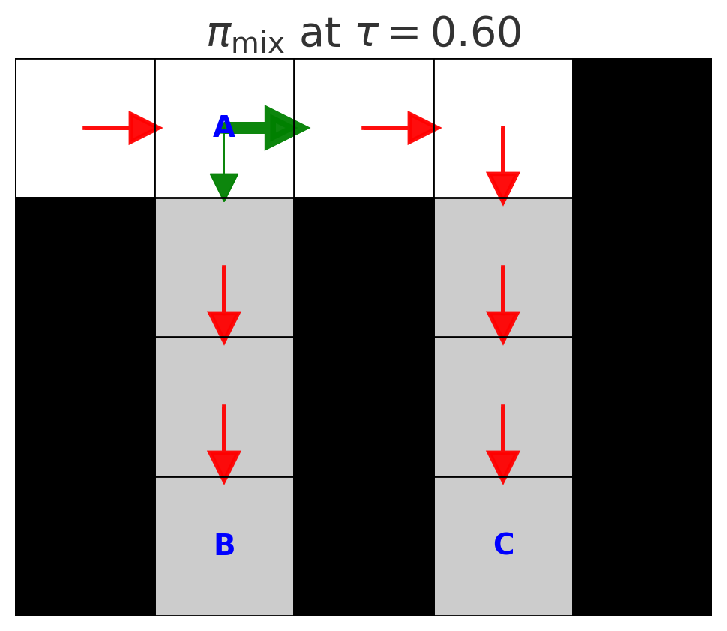}
      \par\vspace{2pt}\footnotesize (b)
    \end{minipage}
    
    \vspace{6pt}
    \begin{minipage}[t]{0.48\linewidth}
      \centering
      \includegraphics[width=\linewidth]{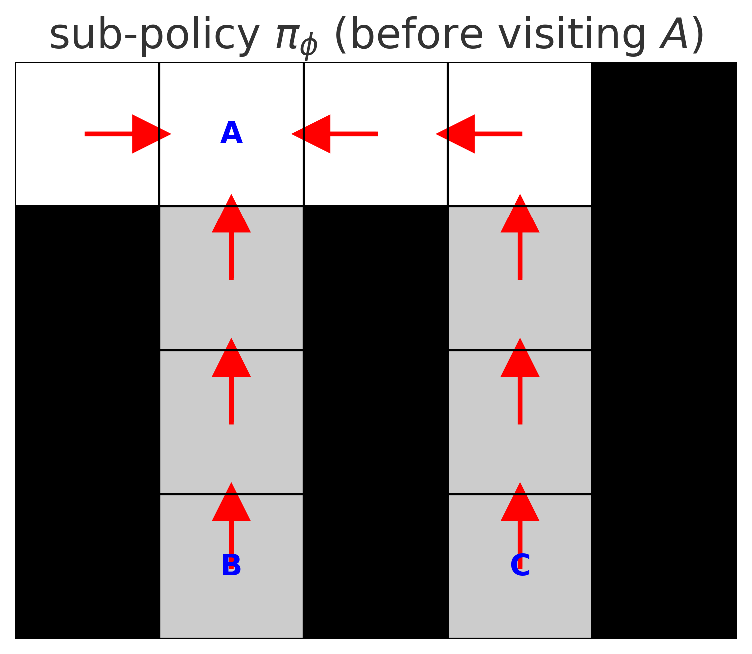}
      \par\vspace{2pt}\footnotesize (c)
    \end{minipage}\hfill
    \begin{minipage}[t]{0.48\linewidth}
      \centering
      \includegraphics[width=\linewidth]{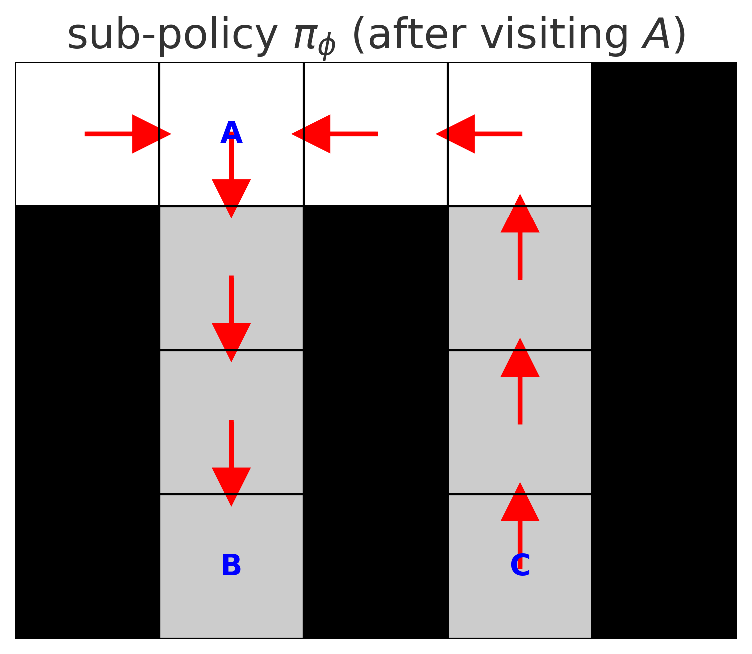}
      \par\vspace{2pt}\footnotesize (d)
    \end{minipage}
    
    \vspace{6pt}
    \begin{minipage}[t]{0.48\linewidth}
      \centering
      \includegraphics[width=\linewidth]{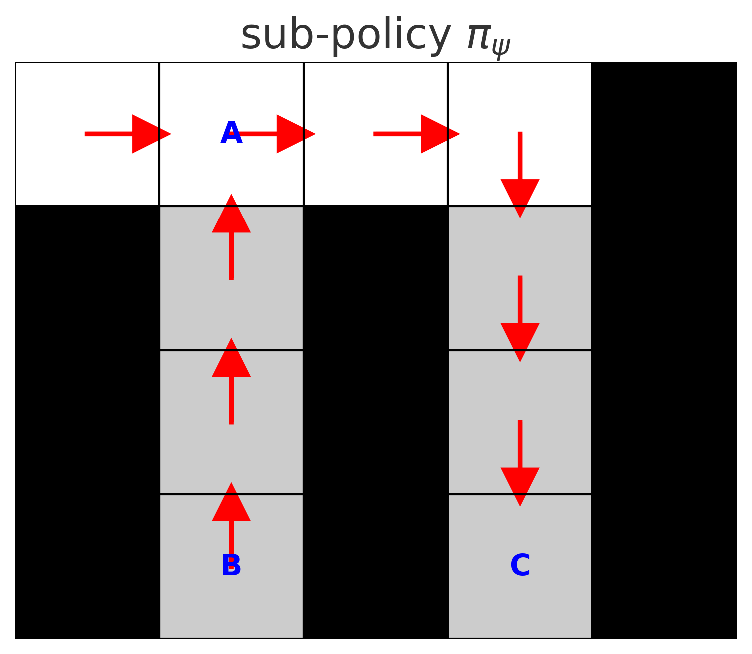}
      \par\vspace{2pt}\footnotesize (e)
    \end{minipage}
    
    \caption{(a)--(b) Pre\mbox{-}switching policy $\pi_{\mathrm{mix}}$ obtained from the linear program at two constraint tolerances ($\tau=0.20$ and $0.60$). 
    Randomization occurs only at $(0,1)$. 
    As $\tau$ increases, the mixing probabilities at $(0,1)$ adjust, while actions elsewhere remain deterministic and unchanged. 
    (c) Sub\mbox{-}policy $\pi_\phi$ before visiting $A$ (move toward $A$). 
    (d) Sub\mbox{-}policy $\pi_\phi$ after visiting $A$ (move toward $B$). 
    (e) Sub\mbox{-}policy $\pi_\psi$ (reach $C$). 
    Red arrows denote deterministic actions; green arrows denote mixed actions with line width proportional to mixing probabilities. 
    Together these sub\mbox{-}policies realize the switching policy $\pi_{s\to s}$, whose illustration is given in Fig.~\ref{fig:switch-policy}.}
    \label{fig:mix_and_subpolicies}
\end{figure}

\begin{figure}[t]
  \centering
  \includegraphics[width=0.9\linewidth]{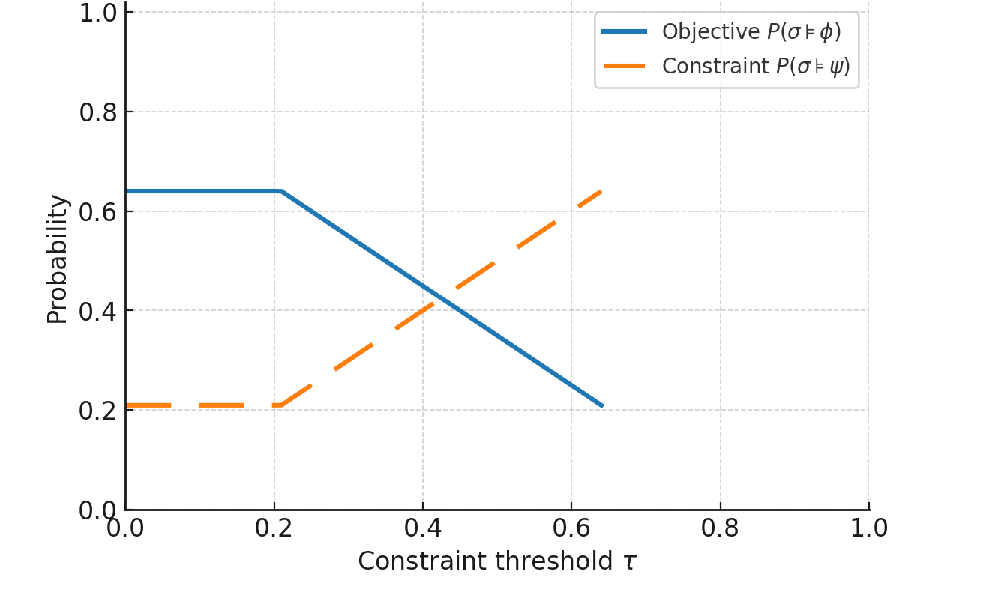}
  \caption{Trade\mbox{-}off between the objective and the constraint as the threshold $\tau$ varies.
  Curves are computed via Algorithm~\ref{alg:scLTL} for $\tau\in[0,1]$ with a step size of $0.01$.
  The problem is feasible up to $\tau=0.64$; beyond this, no policy can satisfy the constraint.
  For small thresholds (approximately $\tau\le 0.21$), the solution coincides with the unconstrained deterministic policy $\pi_\phi$.
  As $\tau$ increases, the policy shifts to satisfy $\psi$, trading off the performance in satisfying $\phi$.}
  \label{fig:tradeoff-lp}
\end{figure}

\subsection{Complexity and Runtime}
The complexity of Algorithm~\ref{alg:scLTL} is polynomial in the number of
decision variables and constraints. 
We denote this by $\mathrm{poly}(m,n)$.
For the double product MDP built from an LMDP with size $(|S|,|\mathcal{A}|)$
and DFAs of sizes $|Q_\phi|$ and $|Q_\psi|$, we have
\[
m = |S|\,|\mathcal{A}|\,|Q_\phi|\,|Q_\psi|,\qquad
n = |S|\,|Q_\phi|\,|Q_\psi|.
\]

Algorithm~\ref{alg:scLTL} computes the mixed, optimistic, and
pessimistic policies. Solving for the mixed policy dominates in cost because it uses more decision variables. 
We solve all LPs with the HiGHS linear optimizer under default settings. 
On our machine, one LP solve per $\tau$ takes $\approx 0.002\,\mathrm{s}$ on average, and sweeping $\tau\in[0,0.64]$ with
step $0.01$ takes $\approx 0.13\,\mathrm{s}$.

\section{Conclusion}
We have addressed the challenging problem of solving the constrained sc\mbox{-}LTL planning problem. 
Our work establishes that, despite the non\mbox{-}Markovian nature of the optimal policy in the original MDP, one can construct an equivalent optimal stationary policy on a double product MDP. 
This structural insight enables us to solve the problem via a linear program over reachability objectives and constraints. 
We demonstrated that switching between different stationary policies suffices to capture the trade\mbox{-}offs between objectives and constraints, and this switching is inherently captured in the double product MDP's structure. 
Our results lay the groundwork for scalable and theoretically sound methods for constrained temporal logic planning in MDPs. 
\section*{Appendix}
In this appendix, we formally prove Lemma~\ref{thm: oc}. 
The proof is based on the fact that all recurrent states in the MDP are contained in end components. 
\begin{definition} 
Let $\mathcal{M} = (S, A, T, \beta)$ be an MDP.  
An \emph{end component} of $\mathcal{M}$ is a pair $(E,A_E)$ such that:
\begin{itemize}
    \item $E \subseteq S$ is non\mbox{-}empty, and $A_E: E \to 2^A$ assigns to each $s\in E$ a non\mbox{-}empty set $A_E(s)\subseteq A(s)$;
    \item Each end component is closed. For all $s \in E$ and $a \in A_E(s)$, every successor $s'$ with $T(s,a,s')>0$ is also in $E$;
    \item Each end component is strongly connected. In the directed graph induced by $E$ and actions $A_E$, every state in $E$ is reachable from every other state in $E$ via actions in $A_E$.
\end{itemize}
\end{definition}

We denote by $\mathrm{EC}(\mathcal{M})$ the set of all end components of $\mathcal{M}$.
Let $\mathrm{EC}_{\pi_h'}(\mathcal{M}) \subseteq \mathrm{EC}(\mathcal{M})$ 
denote the set of end components that are valid under policy $\pi_h'$. 
$\mathrm{EC}_{\pi_h'}(\mathcal{M})$ contains all the recurrent states under policy $\pi_h'$. 

{Throughout this appendix, probabilities are taken with respect to the induced path measure $P_{\pi}$ defined in Section~\ref{sec:prelim}, under the fixed initial distribution $\beta$ of $\mathcal M$.}

\hspace*{1em}\emph{Proof of Lemma~\ref{thm: oc}: }
Applying the flow constraint to $\pi_h'$ gives, for all $x\in S^{?}$,
\begin{align}
\mu_{\pi_h'}(x)
&=\beta(x)
+\sum_{y\in S^{?}}\sum_{a\in A(y)} \mu_{\pi_h'}(y,a)\,T(y,a,x) \notag\\
&=\beta(x)
+\sum_{y\in S^{?}}\sum_{a\in A(y)} \mu_{\pi_h'}(y)\,\pi_{\mathrm{mix}}(y,a)\,T(y,a,x)
\label{eq:flow_hprime_using_pimix}
\end{align}
Similarly, applying the flow constraint to the stationary policy $\pi_{\mathrm{mix}}$ yields
\begin{align}\label{eq:flow_mix}
\mu_{\pi_{\mathrm{mix}}}(x)
=\beta(x)
+\sum_{y\in S^{?}}\sum_{a\in A(y)} \mu_{\pi_{\mathrm{mix}}}(y)\,\pi_{\mathrm{mix}}(y,a)\,T(y,a,x).
\end{align}

Let $S^{?}_{>0}\coloneqq\{x\in S^{?}\mid \mu_{\pi_h'}(x)>0\}$.
For any $x\in S^{?}\setminus S^{?}_{>0}$ we have $\mu_{\pi_h'}(x)=0$.
Moreover, since $\pi_{\mathrm{mix}}$ is constructed by normalizing $\mu_{\pi_h'}(\cdot,\cdot)$ on $S^{?}$,
there is no inflow into such $x$ from any $(y,a)$ with $\mu_{\pi_h'}(y,a)>0$, and hence $\mu_{\pi_{\mathrm{mix}}}(x)=0$ as well.
Therefore it suffices to compare the state occupancy measure on $S^{?}_{>0}$.

In matrix form, let $\mu$ and $\beta$ be row vectors in $\mathbb{R}^{1\times |S^{?}_{>0}|}$ 
\begin{align}
&\mu_{\pi_h'}=\beta +\mu_{\pi_h'}P_{S^{?}_{>0}\rightarrow S^{?}_{>0}}^{\pi_\mathrm{mix}}, \notag \\
&\mu_{\pi_\mathrm{mix}}=\beta +\mu_{\pi_\mathrm{mix}}P_{S^{?}_{>0}\rightarrow S^{?}_{>0}}^{\pi_\mathrm{mix}}, 
\label{eq:two_same_linear_systems}
\end{align}
where $P_{S^{?}_{>0}\rightarrow S^{?}_{>0}}^{\pi_\mathrm{mix}}$ is the substochastic transition submatrix induced by $\pi_\mathrm{mix}$ on $S^{?}_{>0}$.

Moreover, by the construction of $\pi_h'$, the pre\mbox{-}switch occupancies satisfy $\mu_{\pi_h'}(x)<\infty$ for all $x\in S^{?}$, hence in particular for all $x\in S^{?}_{>0}$.
This rules out any recurrent (closed) class contained in $S^{?}_{>0}$ that is reachable from $\beta$ under
$P_{S^{?}_{>0}\rightarrow S^{?}_{>0}}^{\pi_\mathrm{mix}}$. 
Otherwise, with positive probability the path would remain in $S^{?}$ forever,
implying $\mu_{\pi_h'}(x)=\infty$ for states in the class, a contradiction.
Therefore, all states in $S^{?}_{>0}$ are transient under $\pi_\mathrm{mix}$, and thus
$(I-P_{S^{?}_{>0}\rightarrow S^{?}_{>0}}^{\pi_\mathrm{mix}})$ is invertible.

Consequently, both $\mu_{\pi_h'}$ and $\mu_{\pi_{\mathrm{mix}}}$ are the unique solution to the same equation,
\begin{align}
    \mu_{\pi_h'}
    &=\beta (I-P_{  S^{?}_{>0}\rightarrow   S^{?}_{>0}}^{\pi_\mathrm{mix}})^{-1} \notag \\
    &=\mu_{\pi_\mathrm{mix}}.
\end{align}

Finally, for any $x\in S^{?}$ and $a\in A(x)$,
\begin{align}
\mu_{\pi_{\mathrm{mix}}}(x,a)  
&=\mu_{\pi_{\mathrm{mix}}}(x)\pi_{\mathrm{mix}}(x,a) \notag \\
&=\mu_{\pi_h'}(x)\pi_{\mathrm{mix}}(x,a) \notag \\
&=\mu_{\pi_h'}(x,a). 
\end{align}
This completes the proof of Lemma~\ref{thm: oc}. \hfill $\blacksquare$


\end{document}